\documentclass[preprint,authoryear,10pt,a4paper]{elsarticle}
\usepackage[top=0.75in, bottom=0.75in, left=0.75in, right=0.75in]{geometry}

\usepackage{palatino}
\usepackage{amsmath,amsthm}
\usepackage{amssymb}
\usepackage{bbm}
\usepackage{amsfonts}
\usepackage{graphicx}
\usepackage{subfigure}
\graphicspath{{graphics/}}
\usepackage[ruled,vlined]{algorithm2e}
\usepackage{tabularx,booktabs}

\usepackage{xcolor}
\definecolor{darkblue}{rgb}{0.0,0.5,0.5}
\definecolor{blue}{rgb}{0.0,0.5,0.68}
\usepackage[colorlinks]{hyperref}
\hypersetup{colorlinks,breaklinks,linkcolor=darkblue,urlcolor=darkblue,anchorcolor=darkblue,citecolor=darkblue}
\usepackage{booktabs,caption}
\usepackage{multirow}
\usepackage{lscape}
\usepackage{epstopdf}
\usepackage{lineno}
\usepackage{subfigure}

\usepackage{microtype}
\newtheorem{lemma}{Lemma}

\journal{Transportation Research Part C: Emerging Technologies}

\begin{document}

\begin{frontmatter}



\title{{\fontfamily{lmss}\selectfont Scalable Low-Rank Tensor Learning for Spatiotemporal Traffic Data Imputation}}

\author[label]{Xinyu Chen}
\ead{chenxy346@gmail.com}

\author[label2]{Yixian Chen}
\ead{yxnchen.work@gmail.com}

\author[label]{Nicolas Saunier}
\ead{nicolas.saunier@polymtl.ca}

\author[label1]{Lijun Sun\corref{cor1}}
\ead{lijun.sun@mcgill.ca}

\address[label]{Department of Civil, Geological and Mining Engineering, Polytechnique Montreal, Montreal, QC H3T 1J4, Canada}
\address[label2]{School of Intelligent Systems Engineering, Sun Yat-Sen University, Guangzhou 510006, China}
\address[label1]{Department of Civil Engineering, McGill University, Montreal, QC H3A 0C3, Canada}

\cortext[cor1]{Corresponding author. Address: 492-817 Sherbrooke Street West, Macdonald Engineering Building, Montreal, Quebec H3A 0C3, Canada}

\begin{abstract}
Missing value problem in spatiotemporal traffic data has long been a challenging topic, in particular for large-scale and high-dimensional data with complex missing mechanisms and diverse degrees of missingness. Recent studies based on tensor nuclear norm have demonstrated the superiority of tensor learning in imputation tasks by effectively characterizing the complex correlations/dependencies in spatiotemporal data. However, despite the promising results, these approaches do not scale well to large data tensors. In this paper, we focus on addressing the missing data imputation problem for large-scale spatiotemporal traffic data. To achieve both high accuracy and efficiency, we develop a scalable tensor learning model---Low-Tubal-Rank Smoothing Tensor Completion (LSTC-Tubal)---based on the existing framework of Low-Rank Tensor Completion, which is well-suited for spatiotemporal traffic data that is characterized by multidimensional structure of location$\times$ time of day $\times$ day. In particular, the proposed LSTC-Tubal model involves a scalable tensor nuclear norm minimization scheme by integrating linear unitary transformation. Therefore, tensor nuclear norm minimization can be solved by singular value thresholding on the transformed matrix of each day while the day-to-day correlation can be effectively preserved by the unitary transform matrix. Before setting up the experiment, we consider some real-world data sets, including two large-scale 5-minute traffic speed data sets collected by the California PeMS system with 11160 sensors: 1) PeMS-4W covers the data over 4 weeks (i.e., $288\times 28$ time points), and 2) PeMS-8W covers the data over 8 weeks (i.e., $288\times 56$ time points). We compare LSTC-Tubal with state-of-the-art baseline models, and find that LSTC-Tubal can achieve competitively accuracy with a significantly lower computational cost. In addition, the LSTC-Tubal will also benefit other tasks in modeling large-scale spatiotemporal traffic data, such as network-level traffic forecasting.

\end{abstract}

\begin{keyword}
Spatiotemporal traffic data, high-dimensional data, missing data imputation, low-rank tensor completion, linear unitary transformation, quadratic variation
\end{keyword}

\end{frontmatter}

\section{Introduction}

With remarkable advances in low-cost sensing technologies, an increasing amount of real-world traffic measurement data are collected with high spatial sensor coverage and fine temporal resolution. These large-scale and high-dimensional traffic data provide us with unprecedented opportunities for sensing traffic dynamics and developing efficient and reliable applications for intelligent transportation systems (ITS). However, there are two critical issues that undermine the use of these data in real-world applications: (1) the presence of missing and corrupted data remains a primary challenge and makes it difficult to get the true signals, and (2) it is computationally challenging to perform large-scale and high-dimensional traffic data analysis using existing methods/algorithms. The issue of missing and corrupted data may arise from complicated sensor malfunctioning, communication failure, and even unsatisfied sensor coverage. To improve data quality and support downstream applications, missing data imputation becomes an essential task. Despite recent efforts and advances in developing machine learning-based imputation models, most studies still focus on small-scale data sets (e.g., 100$\sim$400 time series) and it remains a critical challenge to perform accurate and efficient imputation on large-scale traffic data sets.

In the literature, there are numerous methods for modeling corrupted spatiotemporal traffic data. We summarize these methods from the aspect of time series modeling in two main categories. The first category focuses on single time series modeling, which includes regression methods such as linear regression and autoregressive model \citep{Schafer1997Analysis, chen2000nearest}. Although this approach is simple, it cannot characterize the correlations/dependencies among different sensors in a spatiotemporal setting. To address this problem, the second category focuses on multivariate/multidimensional setting by modeling multivariate time series data using matrix/tensor structures. In this category, important frameworks include: 1) classical time series analysis methods like vector autoregressive model \citep{Bashir2017}, 2) purely low-rank matrix factorization/completion, e.g., low-rank matrix factorization \citep{asif2013low,asif2016matrix}, principle component analysis \citep{qu2008abpca,qu2009ppca} and their variants, 3) purely low-rank tensor factorization/completion, e.g., Bayesian tensor factorization \citep{chen2019abayesian} and low-rank tensor completion (LRTC) \citep{ran2016tensor,chen2020anonconvex}, and 4) low-rank matrix factorization by integrating time series models, e.g., temporal regularized matrix factorization \citep{yu2016temporal} and Bayesian temporal matrix factorization \citep{chen2019bayesian}. For this large family of methods, a common goal is to capture both temporal dynamics and spatial consistency. For example, it is possible to report superior performance by imposing both local consistency and global consistency in traffic series data \citep{li2015trend}. In addition, there are also some nonlinear methods (e.g., deep learning in \citet{Che2018Recurrent}) that have been applied to address the problem of missing traffic data. Nevertheless, how to generalize these methods to large-scale problems is a long-standing technical gap between research and real-world applications.

Thanks to recent advances in low-rank model, it becomes technically feasible to solve large-scale matrix/tensor learning problems through a series of ``smaller'' subproblems by integrating linear transforms (e.g., discrete Fourier transform, discrete cosine transform) \citep{rao1990dct}. For instance, to learn from partially observed tensor in large-scale image/video data, recent studies have integrated invertible linear transform into LRTC. Linear transforms like Fourier transform and wavelet transform serve as a key role in developing accurate and fast multidimensional tensor completion models \citep{kernfeld2015tensor}. \cite{lu2016tensor,lu2019lowrank,lu2020tensor} introduced some invertible linear transforms (e.g., discrete Fourier transform, discrete cosine transform) into LRTC by connecting tensor tubal rank with tensor singular value decomposition (SVD). As demonstrated in these work, LRTC with invertible linear transforms outperforms by a large margin over LRTC with weighted tensor nuclear norm proposed by \cite{liu2013tensor}. The discrete Fourier transform is the most commonly-used invertible linear transform in the relevant work. To enhance LRTC algorithms with well-suited linear transform, \cite{song2020robust} introduced a data-driven unitary transform (computed from singular vectors of data) in LRTC framework. In this work, numerical experiments demonstrate that unitary transform performs better than both discrete Fourier transform and discrete wavelet transform for tensor completion tasks.

In modeling spatiotemporal traffic data, the most fundamental assumption is that the data has an inherent low-rank structure either in the matrix or in the tensor form. Therefore, the idea of combining tensor completion and linear transform is also well-suited for spatiotemporal traffic data because day-to-day (or week-to-week) correlation of the multivariate traffic time series can be effectively encoded by linear transforms. Inspired by the idea of linear transforms, we can first convert large-scale tensor completion into a series of ``small'' daily (or weekly) traffic data imputation problems and then connect their solutions by linear transforms, thus overcoming the scalability challenge. Following this idea, in this work, we introduce a scalable low-rank tensor learning model to impute missing traffic data. Our contribution is three-fold:

\begin{itemize}
    \item We develop a Low-Tubal-Rank Smoothing Tensor Completion (LSTC-Tubal) model by integrating linear unitary transforms, which is both accurate and efficient for large-scale and high-dimensional traffic data imputation.
    \item We provide new insight on spatiotemporal data modeling where the large-scale data tensor problem can be solved equivalently through a series of ``small'' subproblems by introducing linear transform.
    \item We conduct extensive imputation experiments on some real-world large-scale traffic data sets. The results show that the proposed LSTC-Tubal model is much more efficient than state-of-the-art models while maintaining comparable accuracy.
\end{itemize}

The remainder of this paper is organized as follows. We introduce notations and problem definition in Section~\ref{preliminaries}. Section~\ref{mothodology} introduces in detail the proposed low-rank tensor learning model. In Section~\ref{experiments}, we conduct numerical experiments on some large-scale traffic data sets and show the performance of LSTC-Tubal compared with several state-of-the-art imputation models. Finally, we summarize the study and discuss future research directions in Section~\ref{conclusion}.

\section{Preliminaries}\label{preliminaries}

In this section, we first introduce some fundamental notations and background. After that, we define the problem of missing data imputation in the transportation context.

\subsection{Notations}

In this work, we use boldface uppercase letters to denote matrices, e.g., $\boldsymbol{X}\in\mathbb{R}^{M\times N}$, boldface lowercase letters to denote vectors, e.g., $\boldsymbol{x}\in\mathbb{R}^{M}$, and lowercase letters to denote scalars, e.g., $x$. Given a matrix $\boldsymbol{X}\in\mathbb{R}^{M\times N}$, we denote the $(m,n)$th entry in $\boldsymbol{X}$ by $x_{m,n}$. The Frobenius norm of $\boldsymbol{X}$ is defined as $\|\boldsymbol{X}\|_{F}=\sqrt{\sum_{m,n}x_{m,n}^{2}}$, and the $\ell_2$-norm of $\boldsymbol{x}$ is defined as $\|\boldsymbol{x}\|_{2}=\sqrt{\sum_{m}x_{m}^{2}}$. We denote a third-order tensor by $\boldsymbol{\mathcal{X}}\in\mathbb{R}^{M\times I\times J}$ and the $k$th-mode ($k=1,2,3$) unfolding of $\boldsymbol{\mathcal{X}}$ by $\boldsymbol{\mathcal{X}}_{(k)}$ \citep{kolda2009tensor}. Correspondingly, a folding operator that converts a matrix to a third-order tensor in the $k$th-mode is $\operatorname{fold}_{k}(\cdot)$; thus, we have $\operatorname{fold}_{k}(\boldsymbol{\mathcal{X}}_{(k)})=\boldsymbol{\mathcal{X}}$. For the tensor $\boldsymbol{\mathcal{X}}\in\mathbb{R}^{M\times I\times J}$, we define its Frobenius norm as $\|\boldsymbol{\mathcal{X}}\|_{F}=\sqrt{\sum_{m,i,j}x_{m,i,j}^{2}}$ and its inner product with another tensor as $\left\langle\boldsymbol{\mathcal{X}},\boldsymbol{\mathcal{Y}}\right\rangle=\sum_{m,i,j}x_{m,i,j}y_{m,i,j}$ where $\boldsymbol{\mathcal{Y}}$ is of the same size as $\boldsymbol{\mathcal{X}}$. We denote by $\boldsymbol{X}_{j}=\boldsymbol{\mathcal{X}}_{:,:,j}\in\mathbb{R}^{M\times I}$ the $j$th frontal slice of tensor $\boldsymbol{\mathcal{X}}\in\mathbb{R}^{M\times I\times J}$.

Unitary transform, as a powerful tool in the digital signal processing area, is defined as the invertible linear transform which includes a category of transform (e.g., discrete cosine/sine transform) satisfying the rules of orthogonality and normalization \citep{rao1990dct}. In this work, $\boldsymbol{\Phi}[\cdot]$ represents a unitary transform. We apply the unitary transform to this work by introducing a unitary matrix. Without loss of generality, we specify that the unitary transform imposed on any tensor $\boldsymbol{\mathcal{X}}\in\mathbb{R}^{M\times I\times J}$ is along the third dimension and assume that unitary transform is acted on real-valued data. Therefore, the forward unitary transform can be defined as follows,
\begin{equation}
    \boldsymbol{\Phi}[\boldsymbol{\mathcal{X}}]\equiv\operatorname{fold}_{3}(\boldsymbol{\Phi}^{\top}\boldsymbol{\mathcal{X}}_{(3)})\in\mathbb{R}^{M\times I\times J},
\end{equation}
where $\boldsymbol{\Phi}\in\mathbb{R}^{J\times J}$ is the unitary matrix. Correspondingly, the inverse unitary transform for any given tensor $\boldsymbol{\mathcal{X}}\in\mathbb{R}^{M\times I\times J}$ is
\begin{equation}
    \boldsymbol{\Phi}^{-1}[\boldsymbol{\mathcal{X}}]\equiv\operatorname{fold}_{3}(\boldsymbol{\Phi}\boldsymbol{\mathcal{X}}_{(3)})\in\mathbb{R}^{M\times I\times J}.
\end{equation}

\begin{figure}
\centering
  \includegraphics[width=0.8\textwidth]{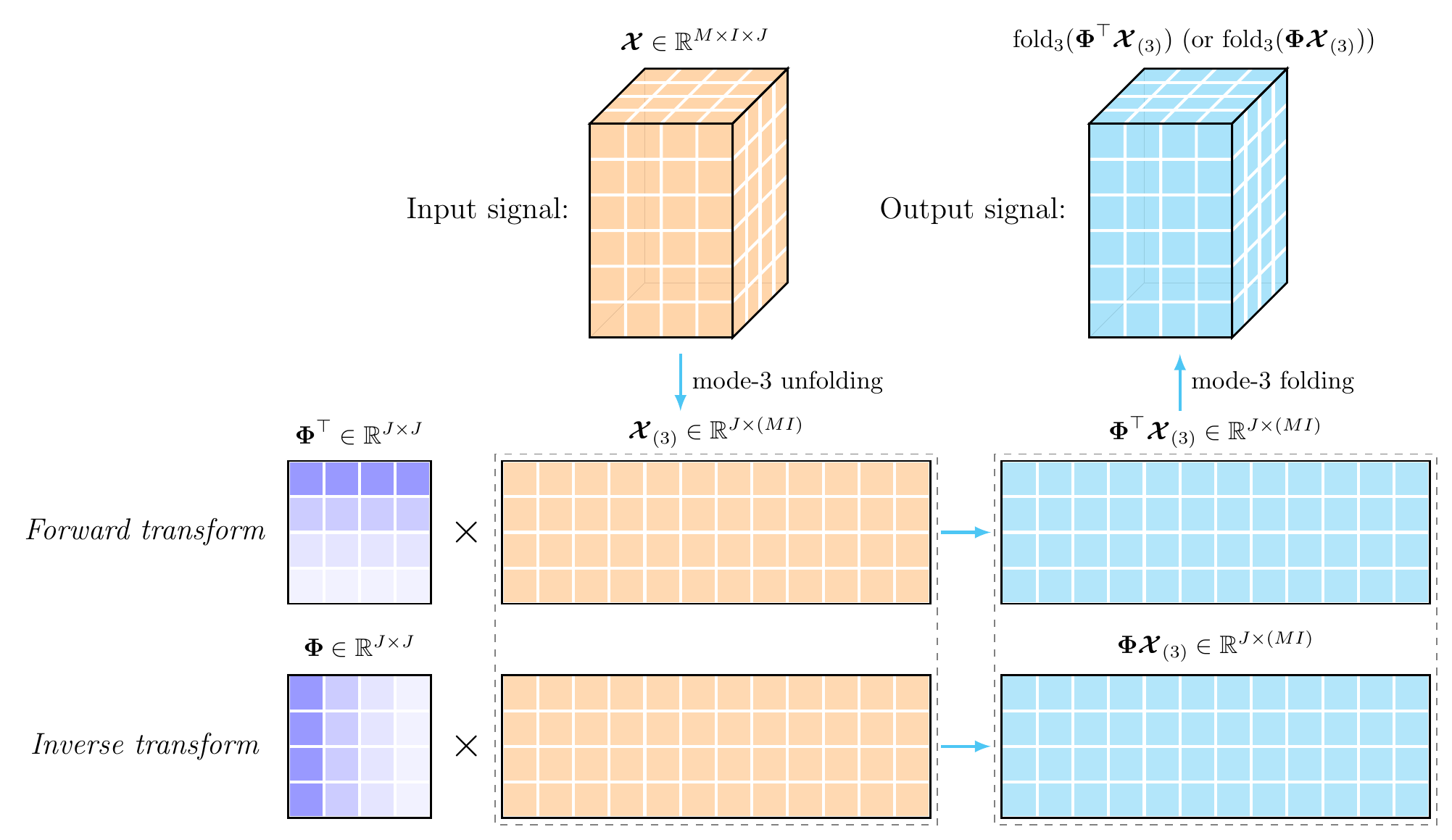}
\caption{Schematic of the forward and inverse unitary transforms along the third dimension of $\boldsymbol{\mathcal{X}}\in\mathbb{R}^{M\times I\times J}$ where $\boldsymbol{\Phi}\in\mathbb{R}^{J\times J}$ is the unitary transform matrix.}
\label{unitary_transform}
\end{figure}

Putting the forward and inverse unitary transforms together, there holds $\boldsymbol{\mathcal{X}}=\boldsymbol{\Phi}^{-1}[\boldsymbol{\Phi}[\boldsymbol{\mathcal{X}}]]$. Figure~\ref{unitary_transform} intuitively shows how to implement forward and inverse unitary transforms along the third dimension of $\boldsymbol{\mathcal{X}}\in\mathbb{R}^{M\times I\times J}$. In this work, we generate the unitary matrix $\boldsymbol{\Phi}$ by the SVD of mode-3 unfolding of the data tensor, and \cite{song2020robust} called the derived matrix as data-driven unitary transform matrix.

\subsection{Problem Definition}

In this work, we introduce spatiotemporal missing time series data imputation tasks in a general sense. For any given partially observed data matrix $\boldsymbol{Y}$ whose columns correspond to time points and rows correspond to spatial locations/sensors:
\begin{equation}
    \boldsymbol{Y}=\left[\begin{array}{cccc}
    \mid & \mid & & \mid \\
    \boldsymbol{y}_{1} & \boldsymbol{y}_{2} & \cdots & \boldsymbol{y}_{IJ} \\
    \mid & \mid & & \mid
    \end{array}\right]\in\mathbb{R}^{M\times (IJ)},
\end{equation}
where $M$ is the number of spatial locations/sensors and $IJ$ is the number of continuous time points. Then the observed values of $\boldsymbol{Y}$ can be written as $\mathcal{P}_{\Omega}(\boldsymbol{Y})$ in which the operator $\mathcal{P}_{\Omega}:\mathbb{R}^{M\times (IJ)}\mapsto\mathbb{R}^{M\times (IJ)}$ is an orthogonal projection supported on the observed index set $\Omega$:
\begin{equation} \notag
    [\mathcal{P}_{\Omega}(\boldsymbol{Y})]_{m,n}=\left\{\begin{array}{ll}
    y_{m,n},     & \text{if $(m,n)\in\Omega$,}  \\
    0,     & \text{otherwise}, \\
    \end{array}\right.
\end{equation}
where $m=1,\ldots,M$ and $n=1,\ldots,IJ$. As shown in Figure~\ref{framework}, the modeling target of missing data imputation can be described as learning the unobserved values from partially observed $\mathcal{P}_{\Omega}(\boldsymbol{Y})$.


In particular, to take advantage of the algebraic structure of tensor, we introduce a forward tensorization operator $\mathcal{Q}(\cdot)$ that converts the spatiotemporal time series matrix into a third-order tensor (see Figure~\ref{framework}). This operation is implemented by splitting the time dimension into (time point per day, day)-indexed combinations. For example, we can generate a third-order tensor by the forward tensorization operator as $\boldsymbol{\mathcal{X}}=\mathcal{Q}(\boldsymbol{Y})\in\mathbb{R}^{M\times I\times J}$. Conversely, we can also convert the resulted tensor into the original matrix by $\boldsymbol{Y}=\mathcal{Q}^{-1}(\boldsymbol{\mathcal{X}})\in\mathbb{R}^{M\times (IJ)}$ in which $\mathcal{Q}^{-1}(\cdot)$ denotes the inverse operator of $\mathcal{Q}(\cdot)$.

\section{Methodology}\label{mothodology}

In this section, we develop a Low-Tubal-Rank Smoothing Tensor Completion (LSTC-Tubal) model for spatiotemporal traffic data imputation. To overcome the large scale and high dimensionality of real-world traffic data, we propose to integrate linear unitary transform into the LSTC-Tubal model.

\subsection{Model Description}

In fact, spatiotemporal traffic data are multivariate/multidimensional time series associated with both spatial and temporal dimensions. To characterize these data with a suitable algebraic structure, one can represent these data in the form of multivariate time series matrix or multidimensional tensor. Thus, the motivation behind modeling spatiotemporal traffic data in this work is to impose a low-rank structure on the data tensor and build a temporal smoothing process on the data matrix simultaneously. To do so, the essential idea of LSTC-Tubal takes into account both tensor nuclear norm minimization and quadratic variation minimization in a unified tensor completion framework:
\begin{equation}
    \begin{aligned}
    \min _{\boldsymbol{\mathcal{X}},\boldsymbol{Z}}~&\|\boldsymbol{\mathcal{X}}\|_{*}+\frac{\lambda}{2}\sum_{t}\|\boldsymbol{z}_{t}-\boldsymbol{z}_{t-1}\|_{2}^{2} \\ \text { s.t.}~&\left\{\begin{array}{l}\boldsymbol{\mathcal{X}}=\mathcal{Q}\left(\boldsymbol{Z}\right), \\ \mathcal{P}_{\Omega}(\boldsymbol{Z})=\mathcal{P}_{\Omega}(\boldsymbol{Y}), \\ \end{array}\right. \\
    \end{aligned}
    \label{lrtc_ar}
\end{equation}
where $\boldsymbol{Y}\in\mathbb{R}^{M\times (IJ)}$ is the partially observed time series matrix. The variable $\boldsymbol{\mathcal{X}}$ is a low-rank tensor while the variable $\boldsymbol{Z}$ is a time series matrix. In the objective, $\boldsymbol{z}_{t-1}$ and $\boldsymbol{z}_{t}$ are the $t-1$th and $t$th columns of the matrix $\boldsymbol{Z}$, respectively. The positive weight parameter $\lambda$ is a trade-off between the tensor nuclear norm and the quadratic variation. We can see intuitively from Figure~\ref{framework} that the framework are able to reconstruct $\boldsymbol{Y}$ with low-rank patterns and time series smoothness because the established constraint $\boldsymbol{\mathcal{X}}=\mathcal{Q}(\boldsymbol{Z})$ is closely related to the partially observed matrix $\boldsymbol{Y}$.

In the objective, we also strengthen the tensor completion model by providing a rigorous connection between tensor tubal rank and tensor nuclear norm (denote by $\|\cdot\|_{*}$) through invertible linear transform. Notably, under invertible linear transform, tensor nuclear norm minimization has been verified as the convex surrogate to the tensor tubal rank minimization \citep{kernfeld2015tensor,lu2016tensor}.

\begin{figure}
\centering
  \includegraphics[width=0.8\textwidth]{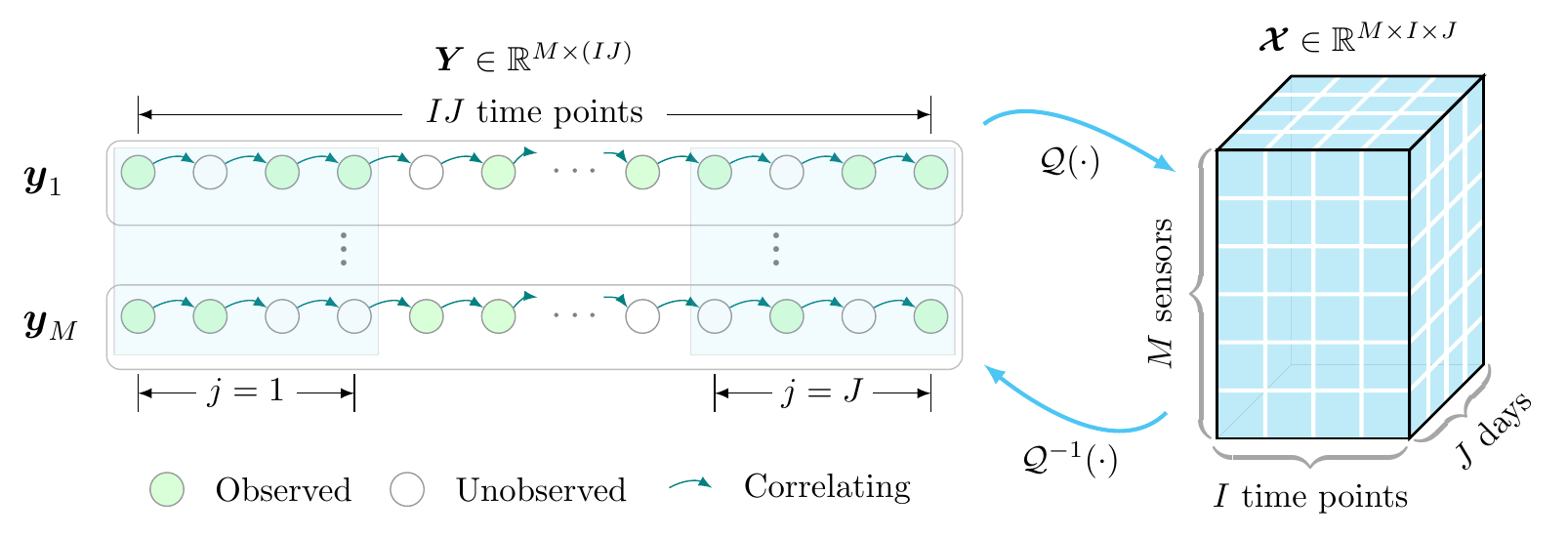}
\caption{Illustration of the proposed LSTC-Tubal framework for spatiotemporal traffic data imputation.}
\label{framework}
\end{figure}

To solve the optimization in Eq.~\eqref{lrtc_ar}, we follow the general idea of tensor completion \citep{liu2013tensor} and use the Alternating Direction Method of Multipliers (ADMM) algorithm: Starting with some given initial point $(\boldsymbol{\mathcal{X}}^{0},\boldsymbol{Z}^{0})$, we generate a sequence $\{(\boldsymbol{\mathcal{X}}^{\ell},\boldsymbol{Z}^{\ell})\}_{\ell\in\mathbb{N}}$ via the scheme. In this case, the augmented Lagrangian function can be written as follows:
\begin{equation}\label{augmented_lagrangian_func}
    \mathcal{L}(\boldsymbol{\mathcal{X}},\boldsymbol{Z},\boldsymbol{\mathcal{T}})=\|\boldsymbol{\mathcal{X}}\|_{*}+\frac{\lambda}{2}\sum_{t}\|\boldsymbol{z}_{t}-\boldsymbol{z}_{t-1}\|_{2}^{2}+\frac{\rho}{2}\|\boldsymbol{\mathcal{X}}-\mathcal{Q}(\boldsymbol{Z})\|_{F}^{2}+\big\langle\boldsymbol{\mathcal{X}}-\mathcal{Q}(\boldsymbol{Z}),\boldsymbol{\mathcal{T}}\big\rangle,
\end{equation}
where $\rho$ is the learning rate of ADMM algorithm, $\big\langle\cdot,\cdot\big\rangle$ denotes the inner product, and $\boldsymbol{\mathcal{T}}\in\mathbb{R}^{M\times I\times J}$ is the dual variable. In particular, $\mathcal{P}_{\Omega}(\boldsymbol{Z})=\mathcal{P}_{\Omega}(\boldsymbol{Y})$ is the fixed constraint for maintaining observation information. According to the augmented Lagrangian function, ADMM can transform the problem in Eq.~\eqref{lrtc_ar} into the following subproblems in an iterative manner:
\begin{align}
    \boldsymbol{\mathcal{X}}^{\ell+1}:&=\operatorname{arg}\min_{\boldsymbol{\mathcal{X}}}~\mathcal{L}(\boldsymbol{\mathcal{X}},\boldsymbol{Z}^{\ell},\boldsymbol{\mathcal{T}}^{\ell}),\label{sub1} \\
    \boldsymbol{Z}^{\ell+1}:&=\operatorname{arg}\min_{\boldsymbol{Z}}~\mathcal{L}(\boldsymbol{\mathcal{X}}^{\ell+1},\boldsymbol{Z},\boldsymbol{\mathcal{T}}^{\ell}),\label{sub2} \\
    \boldsymbol{\mathcal{T}}^{\ell+1}:&=\boldsymbol{\mathcal{T}}^{\ell}+\rho(\boldsymbol{\mathcal{X}}^{\ell+1}-\mathcal{Q}(\boldsymbol{Z}^{\ell+1})),\label{sub3}
\end{align}
where $\ell$ denotes the count of iteration. In what follows, we discuss and show the detailed solutions to Eqs.~\eqref{sub1} and \eqref{sub2}.

\subsection{Computing the Variable $\boldsymbol{\mathcal{X}}$}

In the formulated model, the variable $\boldsymbol{\mathcal{X}}$ is in the form of tensor. With respect to the variable $\boldsymbol{\mathcal{X}}$, the expression in Eq.~\eqref{augmented_lagrangian_func} can be simplified to a standard tensor nuclear norm minimization. Thus, Eq.~\eqref{sub1} yields
\begin{equation}\label{sub1_formula}
\begin{aligned}
    \boldsymbol{\mathcal{X}}^{\ell+1}:&=\operatorname{arg}\min_{\boldsymbol{\mathcal{X}}}~\|\boldsymbol{\mathcal{X}}\|_{*}+\frac{\rho}{2}\left\|\boldsymbol{\mathcal{X}}-\mathcal{Q}(\boldsymbol{Z}^{\ell})\right\|_{F}^{2}+\big\langle\boldsymbol{\mathcal{X}}-\mathcal{Q}(\boldsymbol{Z}^{\ell}),\boldsymbol{\mathcal{T}}^{\ell}\big\rangle \\
    &=\operatorname{arg}\min_{\boldsymbol{\mathcal{X}}}~\|\boldsymbol{\mathcal{X}}\|_{*}+\frac{\rho}{2}\left\|\boldsymbol{\mathcal{X}}-(\mathcal{Q}(\boldsymbol{Z}^{\ell})-\boldsymbol{\mathcal{T}}^{\ell}/\rho)\right\|_{F}^{2}. \\
\end{aligned}
\end{equation}

In this case, the optimization has many solutions, including tensor singular value thresholding for weighted tensor nuclear norm minimization by \cite{liu2013tensor}. One solution by integrating unitary transform can be found in Lemma~\ref{tensor_svt}.

\begin{lemma}\label{tensor_svt}
\citep{song2020robust} For any tensor $\boldsymbol{\mathcal{Z}}\in\mathbb{R}^{M\times I\times J}$, suppose $\boldsymbol{\Phi}\in\mathbb{R}^{J\times J}$ is a unitary transform matrix and the SVD of the unitary transformed matrix $\boldsymbol{\Phi}_{j}[\boldsymbol{\mathcal{Z}}]$ (i.e., the $j$th frontal slice of the unitary transformed tensor $\boldsymbol{\Phi}[\boldsymbol{\mathcal{Z}}]$) is
\begin{equation}
    \boldsymbol{\Phi}_{j}[\boldsymbol{\mathcal{Z}}]=\boldsymbol{U}_{\boldsymbol{\Phi}}\boldsymbol{\Sigma}_{\boldsymbol{\Phi}}\boldsymbol{V}_{\boldsymbol{\Phi}}^{\top},
\end{equation}
where $\boldsymbol{U}_{\boldsymbol{\Phi}}\in\mathbb{R}^{M\times R},\boldsymbol{\Sigma}_{\boldsymbol{\Phi}}\in\mathbb{R}^{R\times R},\boldsymbol{V}_{\boldsymbol{\Phi}}\in\mathbb{R}^{I\times R}$ ($R=\min\{M,I\}$), then an optimal solution to the following optimization problem
\begin{equation}
    \min_{\boldsymbol{\mathcal{X}}}~\|\boldsymbol{\mathcal{X}}\|_{*}+\frac{\rho}{2}\|\boldsymbol{\mathcal{X}}-\boldsymbol{\mathcal{Z}}\|_{F}^{2},
\end{equation}
is given by the tensor singular value thresholding:
\begin{equation}\label{tensor_svt_solution}
    \hat{\boldsymbol{X}}_{j}:=\boldsymbol{\Phi}^{-1}[\mathcal{D}_{1/\rho}(\boldsymbol{\Phi}_{j}[\boldsymbol{\mathcal{Z}}])],\quad\mathcal{D}_{1/\rho}(\boldsymbol{\Phi}_{j}[\boldsymbol{\mathcal{Z}}])=\boldsymbol{U}_{\boldsymbol{\Phi}}[\boldsymbol{\Sigma}_{\boldsymbol{\Phi}}-1/\rho]_{+}\boldsymbol{V}_{\boldsymbol{\Phi}}^{\top},
\end{equation}
where $\hat{\boldsymbol{X}}_{j}\in\mathbb{R}^{M\times I}$ is the $j$th frontal slice of the estimated tensor $\hat{\boldsymbol{\mathcal{X}}}\in\mathbb{R}^{M\times I\times J}$. The operator $\mathcal{D}_{1/\rho}(\cdot)$ denotes the matrix singular value thresholding associated with parameter $\rho$. The symbol $[\cdot]_{+}$ denotes the positive truncation at 0 which satisfies $[\sigma-1/\rho]_{+}=\max\{\sigma-1/\rho,0\}$.
\end{lemma}

In Lemma~\ref{tensor_svt}, there are two main steps to achieve tensor singular value thresholding: (1) applying matrix singular value thresholding to the forward unitary transformed matrices, and (2) stacking the inverse unitary transformed results of matrix singular value thresholding to the slices of the estimated tensor. It is not difficult to see that if we impose a unitary transform along a certain dimension of the tensor, then the tensor nuclear norm minimization can be performed by solving nuclear norm minimization problems of ``small'' matrices, making it computationally efficient and scalable to large data. In the literature, invertible linear transforms applied to the tensor completion problem have also been proved to satisfy certain tensor operation property \citep{kernfeld2015tensor,lu2016tensor}.

In Figure~\ref{tensor_svt_explained}, we provide an intuitive schematic for illustrating Lemma~\ref{tensor_svt}. We can see that the tensor singular value thresholding is well-suited for spatiotemporal traffic data that characterized by the dimensions of ``sensor'', ``time point of day'', and ``day''. The corresponding subproblems in the form of matrix allow one to implement singular value thresholding on the unitary transformed matrix for each day while the day-to-day correlation can be preserved by the unitary transform matrix.

\begin{figure}
\centering
  \includegraphics[width=0.8\textwidth]{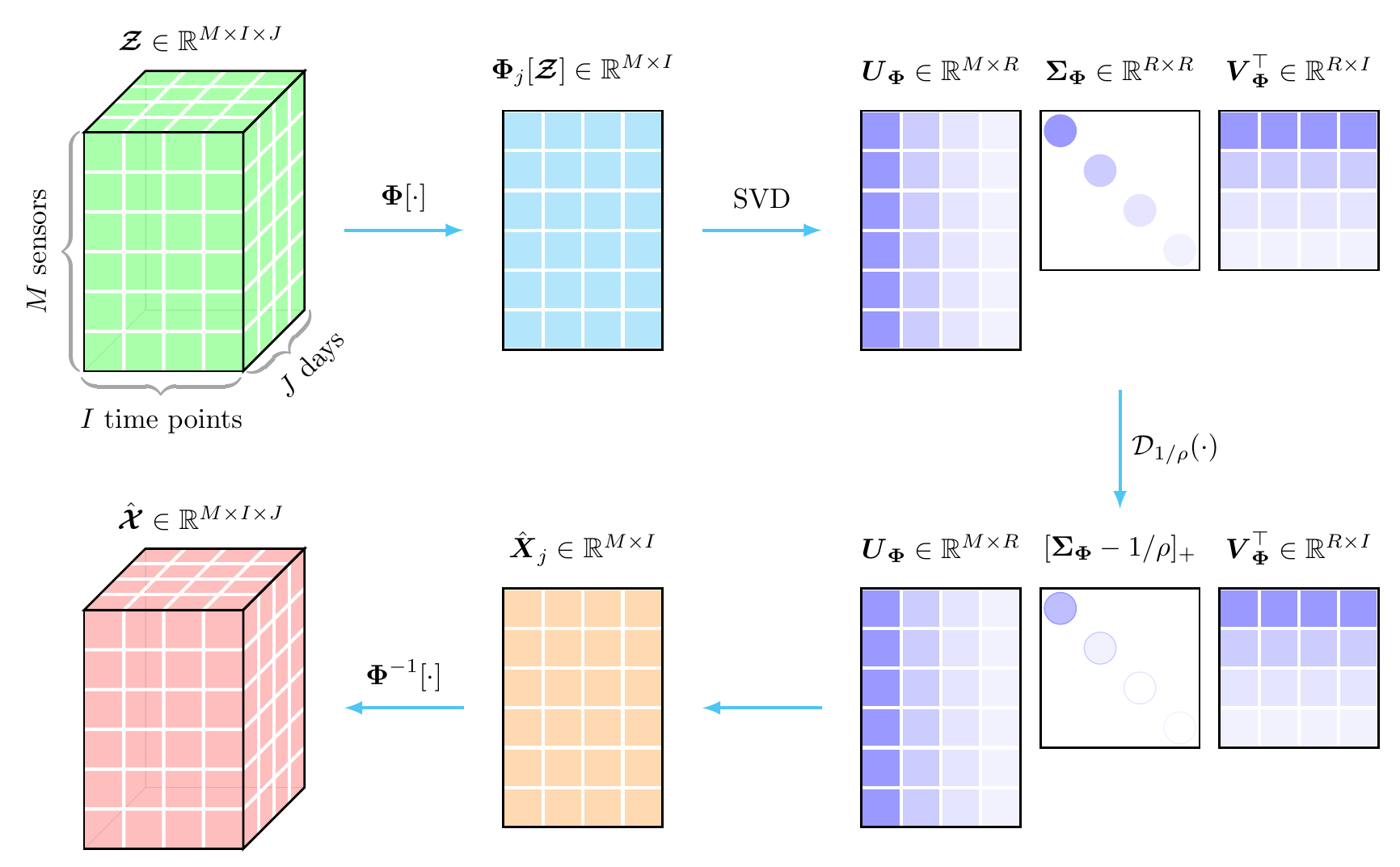}
\caption{Schematic of tensor singular value thresholding ($R=\min\{M,I\}$) with unitary transform. The operators $\boldsymbol{\Phi}[\cdot]$ and $\boldsymbol{\Phi}^{-1}[\cdot]$ denote the forward unitary transform and inverse unitary transform, respectively.}
\label{tensor_svt_explained}
\end{figure}


Therefore, regarding the problem in Eq.~\eqref{sub1_formula}, one can get the closed-form solution by using the tensor singular value thresholding as mentioned in Lemma~\ref{tensor_svt}:
\begin{equation}\label{sub1_solution}
    {\boldsymbol{X}}_{j}^{\ell+1}:=\boldsymbol{\Phi}^{-1}[\mathcal{D}_{1/\rho}(\boldsymbol{\Phi}_{j}[\mathcal{Q}(\boldsymbol{Z}^{\ell})-\boldsymbol{\mathcal{T}}^{\ell}/\rho])],
\end{equation}
where ${\boldsymbol{X}}_{j}^{\ell+1},j=1,\ldots,J$ are the frontal slices of the tensor ${\boldsymbol{\mathcal{X}}}^{\ell+1}$. To get a matrix form solution, we compute $\hat{\boldsymbol{X}}^{\ell+1}=\mathcal{Q}^{-1}({\boldsymbol{\mathcal{X}}}^{\ell+1})\in\mathbb{R}^{M\times(IJ)}$.

\subsection{Computing the Variable $\boldsymbol{Z}$}

The variable $\boldsymbol{Z}$ is in the form of matrix, the formulated quadratic variation on this variable can make the time series smooth. For the fact that $\boldsymbol{\mathcal{X}}=\mathcal{Q}(\boldsymbol{Z})$ shown in Eq.~\eqref{lrtc_ar}, we can rewrite Eq.~\eqref{sub2} with respect to the variable $\boldsymbol{Z}$ as follows,
\begin{equation}\label{sub2_formula}
\begin{aligned}
    \boldsymbol{Z}^{\ell+1}:&=\operatorname{arg}\min_{\boldsymbol{Z}}~\frac{\lambda}{2}\sum_{t}\|\boldsymbol{z}_{t}-\boldsymbol{z}_{t-1}\|_{2}^{2}+\frac{\rho}{2}\left\|\boldsymbol{\mathcal{X}}^{\ell+1}-\mathcal{Q}(\boldsymbol{Z})\right\|_{F}^{2}-\big\langle\mathcal{Q}(\boldsymbol{Z}),\boldsymbol{\mathcal{T}}^{\ell}\big\rangle \\
    &=\operatorname{arg}\min_{\boldsymbol{Z}}~\frac{\lambda}{2}\sum_{t}\|\boldsymbol{z}_{t}-\boldsymbol{z}_{t-1}\|_{2}^{2}+\frac{\rho}{2}\left\|\boldsymbol{Z}-\mathcal{Q}^{-1}(\boldsymbol{\mathcal{X}}^{\ell+1}+\boldsymbol{\mathcal{T}}^{\ell}/\rho)\right\|_{F}^{2}. \\
\end{aligned}
\end{equation}

To solve this optimization problem, we use Lemma~\ref{lemma_tv_solution}.

\begin{lemma}\label{lemma_tv_solution}
For any multivariate time series $\boldsymbol{Z}\in\mathbb{R}^{M\times T}$ consisting of $M$ time series of $T$ consecutive time points, the quadratic variation $\sum_{t=2}^{T}\|\boldsymbol{z}_{t}-\boldsymbol{z}_{t-1}\|_{2}^{2}$ takes the form of
\begin{equation}\label{tv_matrix_form}
    \left\|\left(\boldsymbol{\Psi}_{2}-\boldsymbol{\Psi}_{1}\right)\boldsymbol{Z}^\top\right\|_{F}^{2},
\end{equation}
where $\boldsymbol{\Psi}_{1}=\left[\begin{array}{cc}
  \boldsymbol{I}_{T-1} & \boldsymbol{0}_{T-1} \\
\end{array}\right]\in\mathbb{R}^{(T-1)\times T}$ and $\boldsymbol{\Psi}_{2}=\left[\begin{array}{cc}
  \boldsymbol{0}_{T-1} & \boldsymbol{I}_{T-1} \\
\end{array}\right]\in\mathbb{R}^{(T-1)\times T}$. Note that $\boldsymbol{0}_{T-1}$ is the vector of zeros with length $T-1$, and $\boldsymbol{I}_{T-1}$ is the $T-1$-by-$T-1$ identity matrix.

In this case, an optimal solution to the problem
\begin{equation}
    \min_{\boldsymbol{Z}}~\frac{1}{2}\left\|\left(\boldsymbol{\Psi}_{2}-\boldsymbol{\Psi}_{1}\right)\boldsymbol{Z}^\top\right\|_{F}^{2}+\frac{\alpha}{2}\left\|\boldsymbol{Z}-\boldsymbol{X}\right\|_{F}^{2},
\end{equation}
is given by
\begin{equation}\label{lemma_solution_to_z}
    \hat{\boldsymbol{Z}}=\alpha\boldsymbol{X}\left(\left(\boldsymbol{\Psi}_{2}-\boldsymbol{\Psi}_{1}\right)^\top\left(\boldsymbol{\Psi}_{2}-\boldsymbol{\Psi}_{1}\right)+\alpha\boldsymbol{I}_{T}\right)^{-1}.
\end{equation}
\end{lemma}

\begin{proof}
Suppose $\boldsymbol{\Psi}=\boldsymbol{\Psi}_2-\boldsymbol{\Psi}_1$ and
\begin{equation}
    f(\boldsymbol{Z})=\frac{1}{2}\left\|\boldsymbol{\Psi}\boldsymbol{Z}^\top\right\|_{F}^{2}+\frac{\alpha}{2}\left\|\boldsymbol{Z}-\boldsymbol{X}\right\|_{F}^{2}=\frac{1}{2}\operatorname{tr}\left(\boldsymbol{Z}\boldsymbol{\Psi}^\top\boldsymbol{\Psi}\boldsymbol{Z}^\top\right)+\frac{\alpha}{2}\operatorname{tr}\left((\boldsymbol{Z}-\boldsymbol{X})(\boldsymbol{Z}-\boldsymbol{X})^\top\right),
\end{equation}
where $\operatorname{tr}(\cdot)$ denotes the trace of matrix.

Let $\frac{\partial f(\boldsymbol{Z})}{\partial \boldsymbol{Z}}$ to be $\boldsymbol{0}$:
\begin{equation}
    \frac{\partial f(\boldsymbol{Z})}{\partial \boldsymbol{Z}}=\boldsymbol{Z}(\boldsymbol{\Psi}^\top\boldsymbol{\Psi}+\alpha\boldsymbol{I}_{T})-\alpha\boldsymbol{X}=\boldsymbol{0},
\end{equation}
then, the least-square solution is $\hat{\boldsymbol{Z}}=\alpha\boldsymbol{X}(\boldsymbol{\Psi}^\top\boldsymbol{\Psi}+\alpha\boldsymbol{I}_{T})^{-1}$.
\end{proof}

According to Lemma~\ref{lemma_tv_solution}, one can rewrite the quadratic variation among time snapshots (i.e., $\boldsymbol{z}_{t}$'s) in the form of matrix where $\boldsymbol{\Psi}_{1}$ and $\boldsymbol{\Psi}_{2}$ are fixed parameters. As shown in Eq.~\eqref{lemma_solution_to_z}, the optimization problem has a least-square solution. In light of the above, the solution to Eq.~\eqref{sub2_formula} is
\begin{equation}\label{sub2_solution}
    \boldsymbol{Z}^{\ell+1}:=\frac{\rho}{\lambda}\mathcal{Q}^{-1}(\boldsymbol{\mathcal{X}}^{\ell+1}+\boldsymbol{\mathcal{T}}^{\ell}/\rho)\left(\left(\boldsymbol{\Psi}_{2}-\boldsymbol{\Psi}_{1}\right)^\top\left(\boldsymbol{\Psi}_{2}-\boldsymbol{\Psi}_{1}\right)+\frac{\rho}{\lambda}\boldsymbol{I}_{T}\right)^{-1},
\end{equation}
where the auxiliary variables $\boldsymbol{\Psi}_{1}$ and $\boldsymbol{\Psi}_{2}$ are as same as the ones for Eq.~\eqref{tv_matrix_form}. In fact, the least-square solution in this case might involve the expensive inverse operation on the $T$-by-$T$ matrix because $T$ is a possibly large value. Therefore, it is important to define Eq.~\eqref{sub2_solution} as a sparse linear equation system.

\subsection{Solution Algorithm}

We propose a fast implementation for LSTC-Tubal as shown in Algorithm~\ref{imputer}, which is well-suited to spatiotemporal data imputation problems. In fact, this algorithm does not have too many parameters. Here, the most basic parameter $\rho$ controls the whole ADMM and the tensor singular value thresholding process. The parameter $\lambda$ is a trade-off between tensor nuclear norm and quadratic variation. According to Eq.~\eqref{sub2_solution}, it can be typically set as $\lambda=c\cdot\rho$ where if $c=1$, then it implies that these two parts have the same importance in the result. In this algorithm, some hidden settings are also needed to discuss. We initialize unitary transform matrix $\boldsymbol{\Phi}$ by the left eigenvectors of $\mathcal{Q}(\boldsymbol{{Z}})_{(3)}$. In the iterative process, we update $\boldsymbol{\Phi}$ by the left eigenvectors of $\mathcal{Q}(\boldsymbol{{Z}})_{(3)}-\boldsymbol{\mathcal{T}}_{(3)}/\rho$ for every 10 iterations. At each iteration, recall that the recovered matrix is computed by $\hat{\boldsymbol{X}}^{\ell}=\mathcal{Q}^{-1}(\mathcal{\boldsymbol{X}}^{\ell})$, and if the stopping criteria meets, the algorithm would return the converged $\hat{\boldsymbol{X}}$ as the final result.


\begin{algorithm}[!ht]
\caption{$\text{imputer}(\boldsymbol{Y},\rho,\lambda)$}
\label{imputer}
Initialize $\boldsymbol{\mathcal{T}}^{0}$ as zeros. Set $\mathcal{P}_{\Omega}(\boldsymbol{Z}^{0})=\mathcal{P}_{\Omega}(\boldsymbol{Y})$ and $\ell=0$. \\
\While{not converged}{
Update $\rho$ by $\rho=\min\{1.05\times\rho, \rho_{\text{max}}\}$; \\
\For{$j=1$ \KwTo $J$}{
Update $\boldsymbol{X}_{j}^{\ell+1}$ by Eq.~\eqref{sub1_solution};
}
Update $\boldsymbol{Z}^{\ell+1}$ by Eq.~\eqref{sub2_solution}; \\
Update $\boldsymbol{\mathcal{T}}^{\ell+1}$ by Eq.~\eqref{sub3}; \\
Transform the observation information by letting $\mathcal{P}_{\Omega}(\boldsymbol{Z}^{\ell+1})=\mathcal{P}_{\Omega}(\boldsymbol{Y})$; \\
$\ell:=\ell+1$;
}
\Return recovered matrix $\hat{\boldsymbol{X}}$.
\end{algorithm}

\section{Experiments}\label{experiments}

In this section, we evaluate the imputation performance of the proposed model on some real-world traffic data sets. We measure the model by both imputation accuracy and computational cost.

\subsection{Traffic Data Sets}

To show the advantages of LSTC-Tubal for handling large-scale traffic data, we choose two large-scale data sets collected by the California department of transportation through their Performance Measurement System (PeMS)\footnote{The data sets are available at \url{https://doi.org/10.5281/zenodo.3939792}.} and two other publicly available data sets, the London movement speed data set and Guangzhou traffic speed data set, for the experiments:
\begin{itemize}
    \item\textbf{PeMS-4W data set}: This data set contains freeway traffic speed collected from 11160 traffic measurement sensors over 4 weeks (the first 4 weeks in the year of 2018) with a 5-minute time resolution (288 time intervals per day) in California, USA. It can be arranged in a matrix of size $11160\times 8064$ or a tensor of size $11160\times 288\times 28$ according to the spatial and temporal dimensions. Note that this data set contains about 90 million observations.
    \item\textbf{PeMS-8W data set}: This data set contains freeway traffic speed collected from 11160 traffic measurement sensors over 8 weeks (the first 8 weeks in the year of 2018) with a 5-minute time resolution (288 time intervals per day) in California, USA. It can be arranged in a matrix of size $11160\times 16128$ or a tensor of size $11160\times 288\times 56$ according to the spatial and temporal dimensions. Note that this data set contains about 180 million observations.
    \item\textbf{London-1M data set}: This is London movement speed data set that created by Uber movement project\footnote{\url{https://movement.uber.com}}. This data set includes the average speed on a given road segment for each hour of each day over a whole month (April 2019). In this data set, there are about 220,000 road segments. Note that this data sets only includes the hours or a road segment with at least 5 unique trips in that hour. There are up to 73.09\% missing values and most missing values occur during the night. We choose the subset of this raw data set and build a time series matrix of size $35912\times 720$ (or a tensor of size $35912\times 24\times 30$) in which each time series has at least 70\% observations.
    \item\textbf{Guangzhou-2M data set}\footnote{The data set is available at \url{https://doi.org/10.5281/zenodo.1205229}.}: This traffic speed data set was collected from 214 road segments over two months (61 days from August 1 to September 30, 2016) with a 10-minute resolution (144 time intervals per day) in Guangzhou, China. It can be arranged in a matrix of size $214\times 8784$ or a tensor of size $214\times 144\times 61$.
\end{itemize}

It is not difficult to see that PeMS-4W, PeMS-8W, and London-1M data sets are both large-scale and high-dimensional. In what follows, we create two missing patterns, i.e., random missing (RM) and non-random missing (NM), which are same as in our prior work~\citep{chen2020anonconvex}. According to the mechanism of RM and NM patterns, we mask a certain amount of observations as missing values (i.e., 30\%, 70\%) in these data sets, and the remaining partial observations are input data. To assess the imputation performance, we use the actual values of the masked missing entries as the ground truth to compute MAPE and RMSE:
\begin{equation}
    \text{MAPE}=\frac{1}{n}\sum_{i=1}^{n}\left|\frac{y_i-\hat{y}_i}{y_i}\right|\times 100,\quad\text{RMSE}=\sqrt{\frac{1}{n}\sum_{i=1}^{n}(y_i-\hat{y}_{i})^2},
\end{equation}
where $y_i,\hat{y}_{i},i=1,\ldots,n$ are the actual values and estimated/imputed values, respectively.

\subsection{Baseline Imputation Models}

There are few competing methods for comparison in the imputation experiments on the large-scale traffic data sets because most developed imputation algorithms are not suitable for dealing with such ``big'' data. Thus, we compare the proposed LSTC-Tubal model with the following baseline models:
\begin{itemize}
    \item BPMF: Bayesian Probabilistic Matrix Factorization \citep{Salakhutdinov2008bayesian}. This is a fully Bayesian model of the standard matrix factorization using the Markov chain Monte Carlo (MCMC) algorithm.
    \item BGCP: Bayesian Gaussian CP decomposition \citep{chen2019abayesian}.
    \item BATF: Bayesian Augmented Tensor Factorization \citep{chen2019missing}. This is a fully Bayesian model that integrates global mean value, bias vectors, and factor matrices.
    \item HaLRTC: High-accuracy Low-Rank Tensor Completion \citep{liu2013tensor}. This is a classical LRTC model which uses weighted sum of  nuclear norm from unfolded matrices.
    \item LRTC-TNN: Low-Rank Tensor Completion with Truncated Nuclear Norm minimization \citep{chen2020anonconvex}. This model is built on HaLRTC where the weighted tensor nuclear norm is replaced by multiple truncated nuclear norms with a truncation rate parameter. It has been shown to be superior to TRMF \citep{yu2016temporal}, BTMF \citep{chen2019bayesian}, and the two other baseline models also mentioned in this article, i.e., BGCP \citep{chen2019abayesian} and HaLRTC \citep{liu2013tensor}.
    \item LSTC-DCT: We also compare the LSTC model with discrete cosine transform (DCT). When removing the quadratic variation in LSTC-DCT, i.e., $\lambda=0$, this model is indeed tensor nuclear norm minimization with DCT (i.e., TNN-DCT) proposed by \cite{lu2019lowrank}. Note that DCT is a special case of linear unitary transform. When comparing LSTC-DCT to the proposed LSTC-Tubal model, it is possible to see the importance of unitary transform in LSTC-Tubal.
\end{itemize}

To make the comparison fair, we use the following convergence criteria for the LRTC/LSTC models:
\begin{equation}
    \frac{\|\hat{\boldsymbol{\mathcal{X}}}^{\ell+1}-\hat{\boldsymbol{\mathcal{X}}}^{\ell}\|_{F}^{2}}{\|\mathcal{P}_{\Omega}(\boldsymbol{\mathcal{Y}})\|_{F}^{2}}<\epsilon\quad\text{or}\quad\frac{\|\hat{\boldsymbol{X}}^{\ell+1}-\hat{\boldsymbol{X}}^{\ell}\|_{F}^{2}}{\|\mathcal{P}_{\Omega}(\boldsymbol{Y})\|_{F}^{2}}<\epsilon,
\end{equation}
where $\hat{\boldsymbol{\mathcal{X}}}^{\ell+1}$ and $\hat{\boldsymbol{\mathcal{X}}}^{\ell}$ denote the recovered tensors at the $\ell+1$th iteration and $\ell$th iteration, respectively. $\hat{\boldsymbol{X}}^{\ell+1}$ and $\hat{\boldsymbol{X}}^{\ell}$ denote the recovered matrices at the $\ell+1$th iteration and $\ell$th iteration, respectively. In the following experiments, we set $\epsilon=1\times10^{-3}$.

For the BPMF, BGCP, and BATF models, we set the low rank to 10 on both PeMS data sets due to the possibly high computational cost underlying the greater low rank. For all Gibbs sampling based Bayesian models, we set the number of burn-in iterations and Gibbs iterations as 1000 and 200, respectively. In particular, we discuss LSTC-Tubal in the cases of $\lambda=0$ and $\lambda\neq 0$ because putting LSTC-Tubal ($\lambda=0$) and LSTC-Tubal ($\lambda\neq 0$) together can help evaluate the importance of quadratic variations in the framework. Through cross validation, we evaluate both LSTC-DCT and LSTC-Tubal models with the following parameters: 1) For PeMS-4W and PeMS-8W data, we set $\lambda=0.001\times\rho$ where $\rho=0.001$ (RM) and $\rho=0.0001$ (NM); 2) For London-1M data, we set $\lambda=0.001\times\rho$ and $\rho=0.001$; 3) For Guangzhou-2M data, we set we set $\lambda=0.5\times\rho$ and $\rho=0.002$. The Python code is available at \url{https://github.com/xinychen/transdim}.

\subsection{Imputation Results}

\subsubsection{Evaluation on PeMS-4W and PeMS-8W Data}

Table~\ref{table1} shows the imputation performance of LSTC-Tubal and its competing models on PeMS-4W and PeMS-8W data. For RM scenarios on both PeMS-4W and PeMS-8W data sets, the proposed LSTC-Tubal model achieves high accuracy which is close to the best performance achieved by LRTC-TNN. For NM scenarios, LSTC-Tubal is inferior to the LRTC-TNN model due to whole days missing in each time series. By comparing with HaLRTC, we can see that LSTC-Tubal ($\lambda=0$) performs consistently better in all missing scenarios. This clearly shows that tensor nuclear norm minimization with unitary transform is superior to weighted tensor nuclear norm minimization. The comparison results between LSTC-DCT and LSTC-Tubal indicate that data-driven unitary transform is more capable of characterizing day-to-day correlation than discrete cosine transform.

\begin{table}[!ht]
\caption{Performance comparison (in MAPE/RMSE) for RM and NM data imputation tasks on California PeMS-4W and PeMS-8W traffic speed data.}
\label{table1}
\centering
\footnotesize
\begin{tabular}{l|rrrr|rrrr}
\toprule
& \multicolumn{4}{c|}{PeMS-4W} & \multicolumn{4}{c}{PeMS-8W} \\
\cmidrule(lr){2-5}
\cmidrule(lr){6-9}
& 30\%, RM & 70\%, RM & 30\%, NM & 70\%, NM & 30\%, RM & 70\%, RM & 30\%, NM & 70\%, NM \\
\midrule
BPMF & 4.86/4.21 & 4.88/4.22 & 5.18/4.46 & 5.76/5.09 & 5.16/4.43 & 5.18/4.44 & 5.33/4.56 & 5.51/4.72 \\
BGCP & 5.03/4.31 & 5.02/4.31 & 5.38/4.59 & 6.10/5.50 & 5.28/4.52 & 5.27/4.51 & 5.47/4.66 & 5.63/4.82 \\
BATF & 4.91/4.24 & 4.92/4.25 & 5.31/4.55 & 6.33/5.93 & 5.20/4.47 & 5.22/4.47 & 5.41/4.63 & 5.65/4.88 \\
HaLRTC & 1.98/1.73 & 2.84/2.49 & 5.24/4.19 & 7.17/5.19 & 2.07/1.81 & 3.17/2.74 & 4.79/3.98 & 6.06/4.67  \\
LRTC-TNN & \textbf{1.67}/\textbf{1.55} & \textbf{2.32}/\textbf{2.13} & \textbf{4.62}/\textbf{3.94} & \textbf{5.47}/\textbf{4.71} & 1.81/1.66 & 2.57/\textbf{2.29} & \textbf{4.27}/\textbf{3.69} & \textbf{5.31}/\textbf{4.52} \\
LSTC-DCT & 1.72/1.61 & 2.50/2.28 & 6.22/4.94 & 7.35/5.60 & \textbf{1.75}/\textbf{1.64} & 2.54/2.32 & 5.69/4.69 & 6.57/5.18 \\
LSTC-Tubal ($\lambda=0$) & 1.72/1.61 & 2.47/2.27 & 5.59/4.52 & 6.59/5.07 & \textbf{1.75}/\textbf{1.64} & \textbf{2.51}/2.31 & 5.28/4.42 & 5.97/4.77 \\
LSTC-Tubal ($\lambda\neq 0$) & 1.72/1.61 & 2.47/2.27 & 5.59/4.52 & 6.60/5.07 & \textbf{1.75}/\textbf{1.64} & \textbf{2.51}/2.31 & 5.28/4.42 & 5.97/4.77 \\
\bottomrule
\end{tabular}
\end{table}

Of results in Table~\ref{table1}, LRTC-TNN produces comparable imputation accuracy. However, observing the running time of these imputation models in Table~\ref{table2}, it is not difficult to conclude that LRTC-TNN is not well-suited to these large-scale imputation tasks due to the extremely high computational cost. Table~\ref{table2} clearly shows that LSTC-Tubal with unitary transformed tensor nuclear norm minimization scheme is the most computationally efficient model. Therefore, from both Table~\ref{table1} and \ref{table2}, the results suggest that the proposed LSTC-Tubal model is an efficient solution to large-scale traffic data imputation while still maintaining an imputation performance close to the state-of-the-art models.

\begin{table}[!ht]
\caption{Running time (in minutes) of imputation models on California PeMS-4W and PeMS-8W traffic speed data. We also report the number of iterations when the algorithm has converged.}
\label{table2}
\centering
\footnotesize
\begin{tabular}{l|rrrr|rrrr}
\toprule
& \multicolumn{4}{c|}{PeMS-4W} & \multicolumn{4}{c}{PeMS-8W} \\
\cmidrule(lr){2-5}
\cmidrule(lr){6-9}
& 30\%, RM & 70\%, RM & 30\%, NM & 70\%, NM & 30\%, RM & 70\%, RM & 30\%, NM & 70\%, NM \\
\midrule
BPMF & 100.51 & 102.76 & 103.96 & 102.66 & 194.63 & 192.17 & 190.65 & 192.39 \\
BGCP & 287.48 & 286.09 & 284.77 & 285.08 & 585.60 & 553.88 & 563.42 & 551.31 \\
BATF & 473.97 & 444.98 & 474.93 & 444.16 & 934.83 & 864.08 & 947.03 & 906.97 \\
HaLRTC & 88.35 (18) & 215.33 (37) & 176.31 (34) & 321.20 (61) & 245.41 (18) & 505.22 (38) & 473.43 (35) & 785.40 (59) \\
LRTC-TNN & 354.82 (74) & 322.66 (67) & 468.25 (100) & 465.34 (100) & 877.37 (67) & 807.00 (60) & 1302.68 (100) & 1307.53 (100) \\
LSTC-DCT & 6.67 (14) & 8.91 (20) & 21.52 (47) & 25.44 (59) & \textbf{13.36} (14) & \textbf{18.97} (20) & 44.17 (47) & \textbf{53.45} (58) \\
LSTC-Tubal ($\lambda=0$) & \textbf{6.19} (16) & \textbf{9.61} (27) & \textbf{16.34} (48) & \textbf{23.72} (69) & {14.25} (16) & {25.20} (28) & \textbf{39.11} (48) & {59.60} (70) \\
LSTC-Tubal ($\lambda\neq 0$) & 8.16 (16) & 12.83 (27) & 23.43 (48) & 32.13 (69) & 18.65 (16) & 31.90 (28) & 54.72 (48) & 77.90 (74) \\
\bottomrule
\end{tabular}
\end{table}

We further discuss the imputation of the LSTC-Tubal ($\lambda\neq0$) by providing the heatmap of unitary transform matrices in Figure~\ref{unitary_matrices} and the visualization of time series in Figure~\ref{time_series_curves}. As shown in Figure~\ref{unitary_matrices}(a), the initial unitary transform matrix produced by the partially observed tensor only harnesses the most basic daily patterns where the first few columns can help identify weekday and weekend. In the following iterative process, the patterns underlying unitary transform matrix become more complicated and informative (see Figure~\ref{unitary_matrices}(b-c)). Figure~\ref{time_series_curves} shows that LSTC-Tubal can produce masked time series points accurately by learning from partial observations.

\begin{figure*}[!ht]
\centering
\subfigure[The initialized one.]{
    \centering
    \includegraphics[scale=0.43]{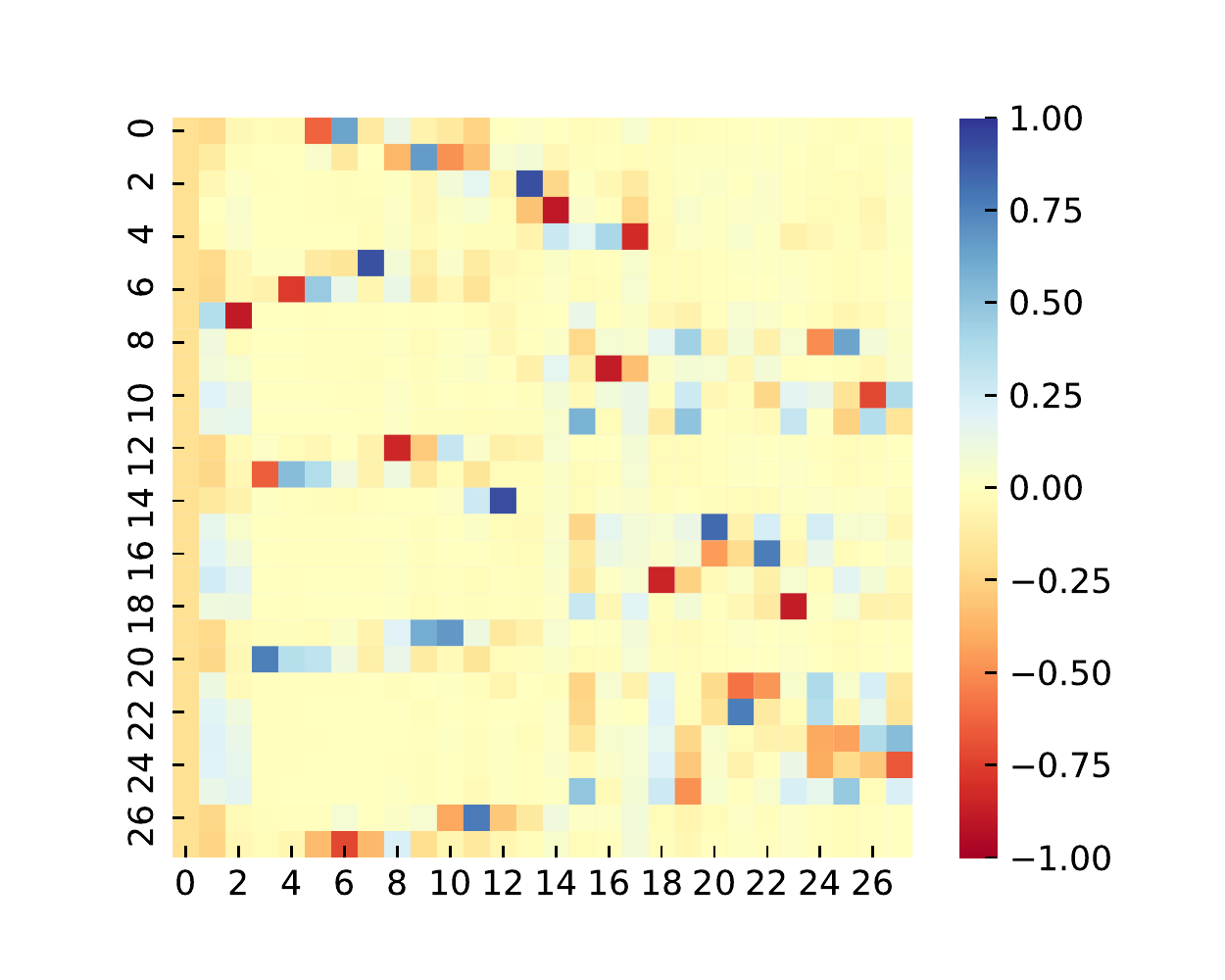}
}
\subfigure[At the 10th iteration.]{
    \centering
    \includegraphics[scale=0.43]{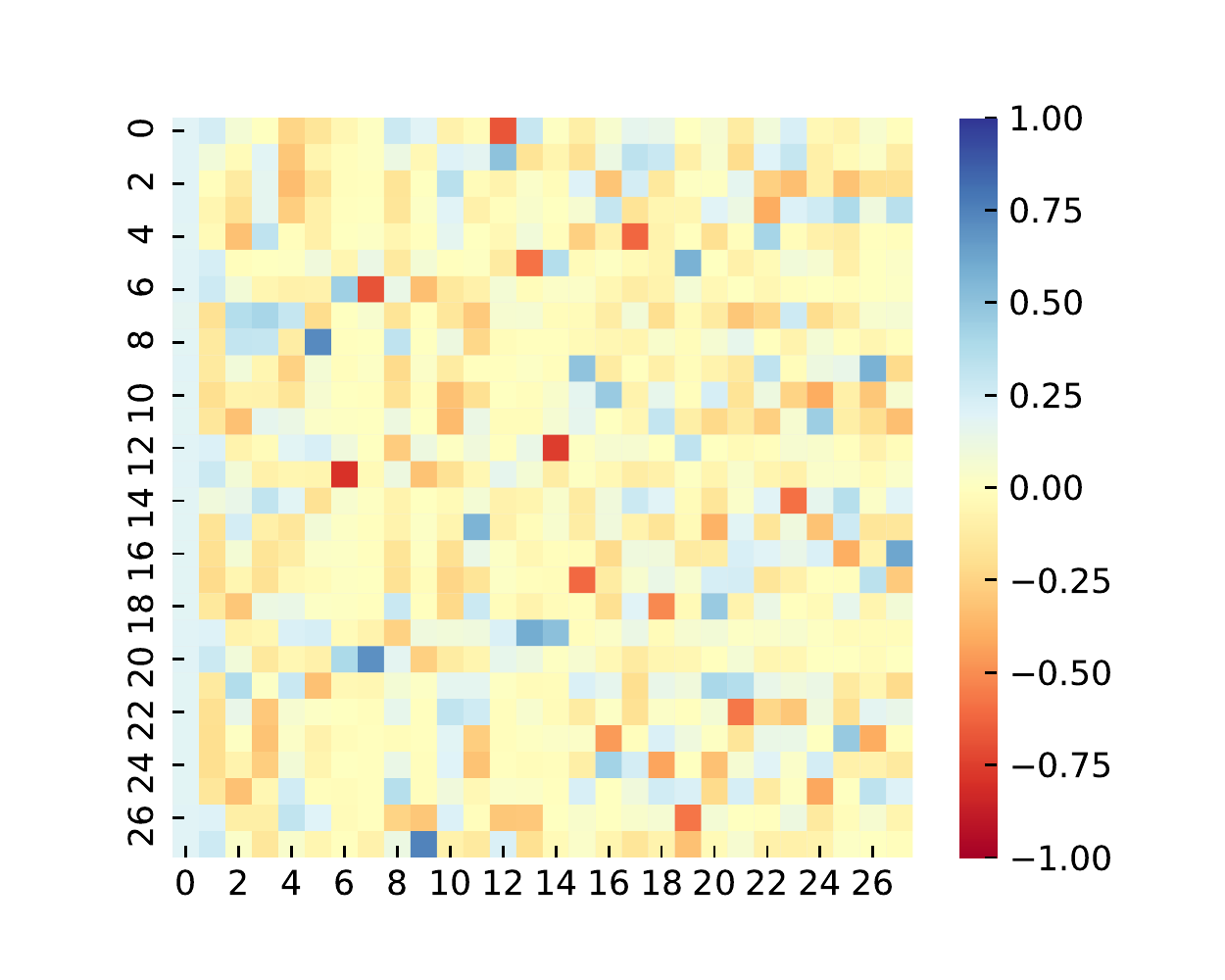}
}
\subfigure[At the 20th iteration.]{
    \centering
    \includegraphics[scale=0.43]{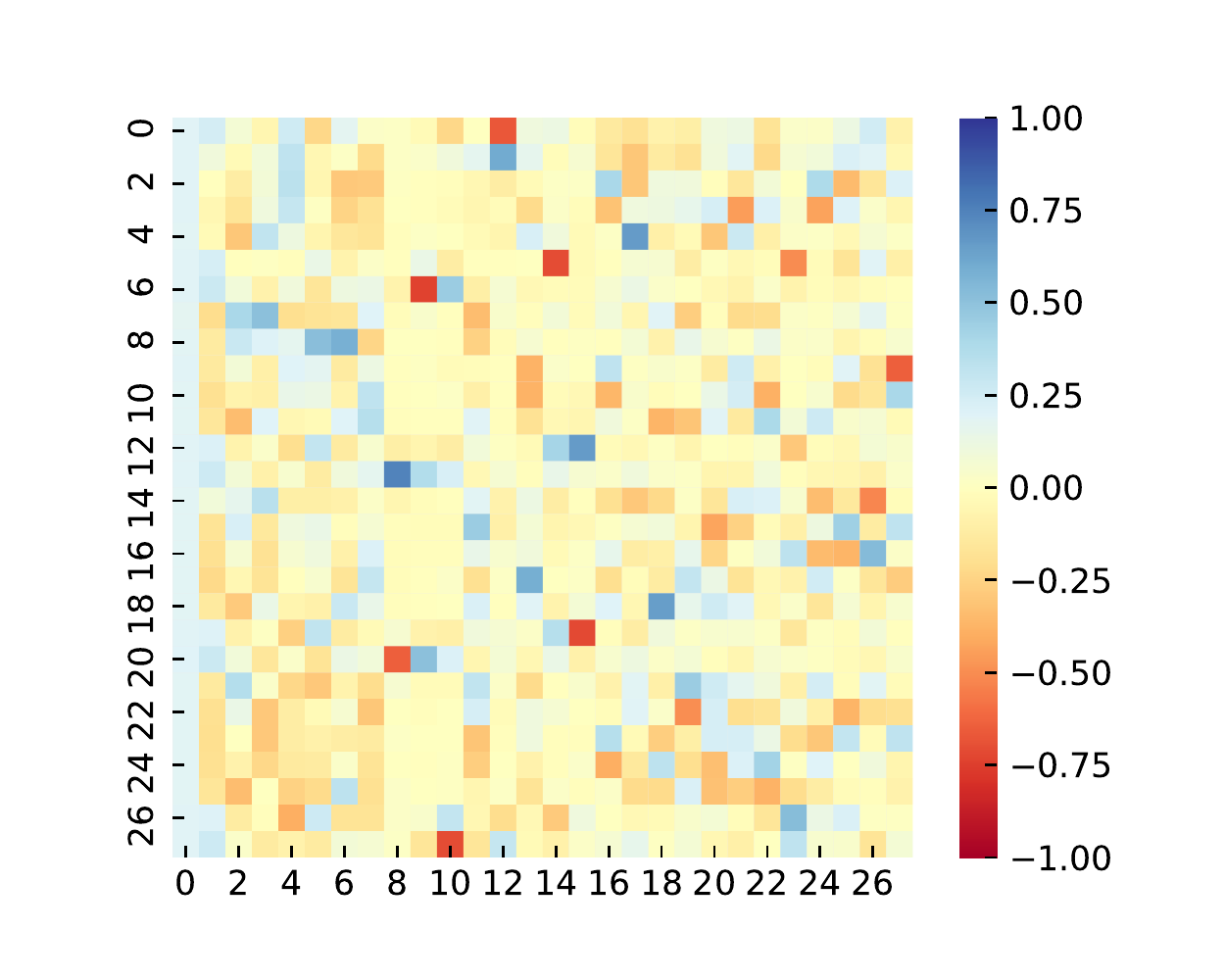}
}
\caption{28-by-28 unitary transform matrices of the PeMS-4W data at the case of 70\% RM scenario.}
\label{unitary_matrices}
\end{figure*}

\begin{figure*}[!ht]
\centering
\subfigure[The 1st time series.]{
    \centering
    \includegraphics[scale=0.4]{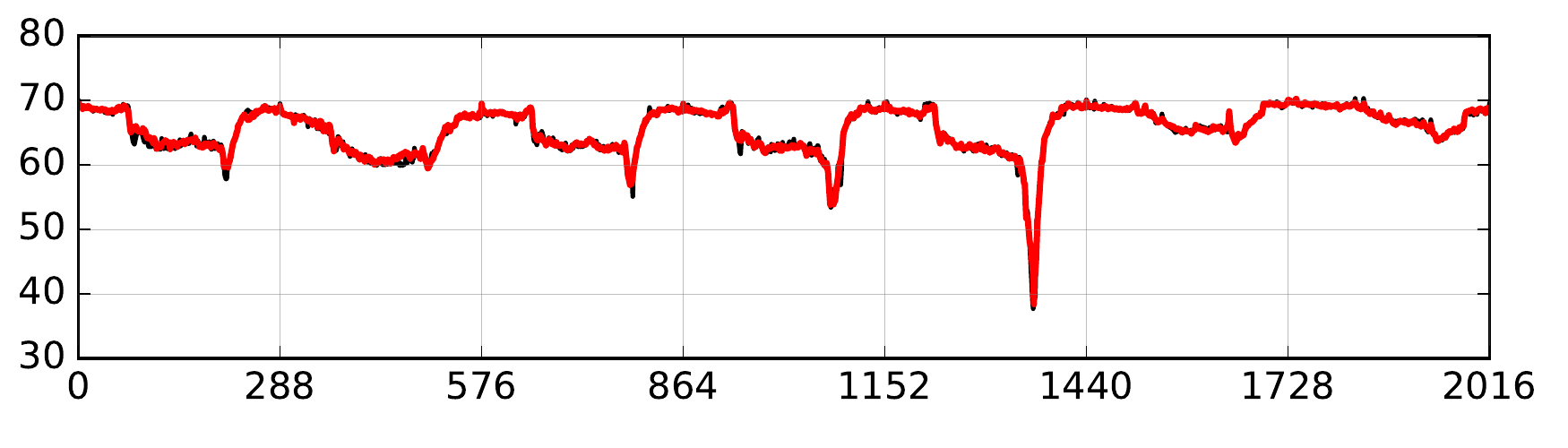}
}
\subfigure[The 2nd time series.]{
    \centering
    \includegraphics[scale=0.4]{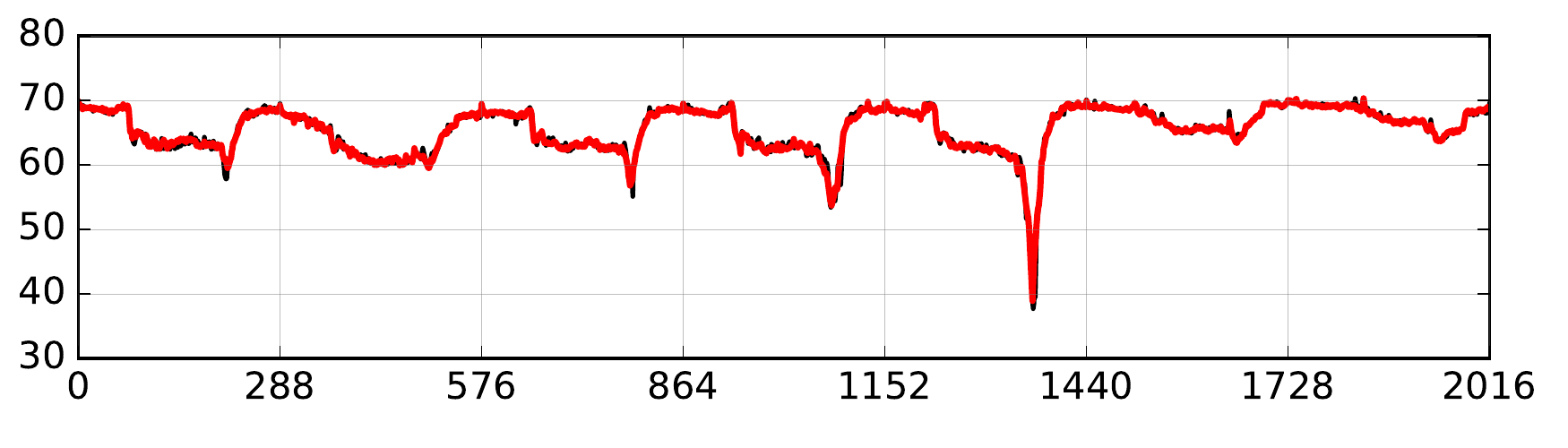}
}
\subfigure[The 3rd time series.]{
    \centering
    \includegraphics[scale=0.4]{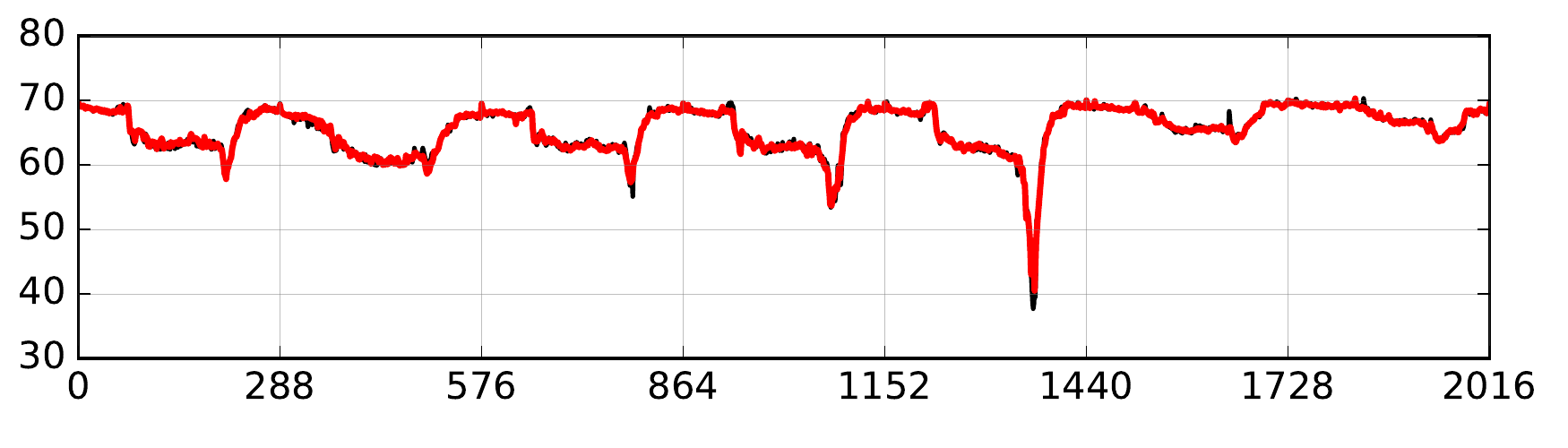}
}
\subfigure[The 4th time series.]{
    \centering
    \includegraphics[scale=0.4]{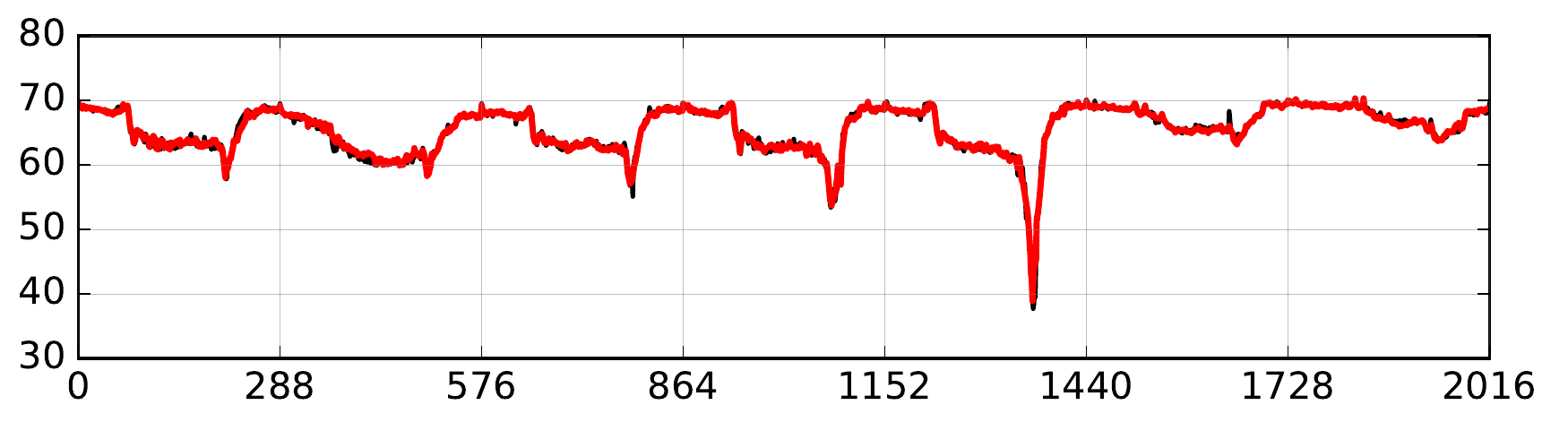}
}
\subfigure[The 5th time series.]{
    \centering
    \includegraphics[scale=0.4]{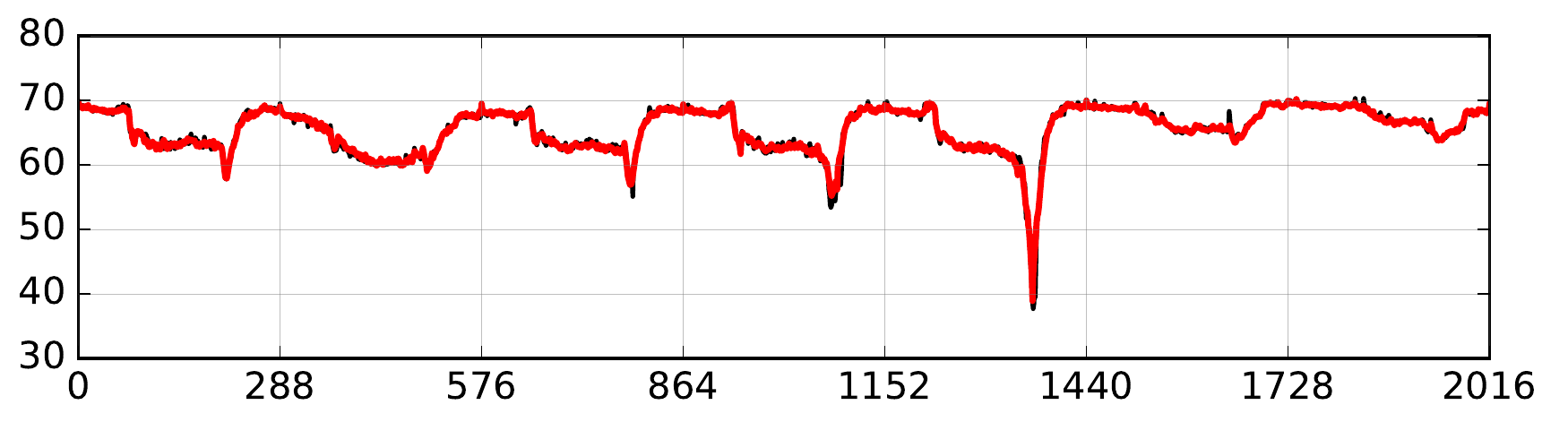}
}
\subfigure[The 6th time series.]{
    \centering
    \includegraphics[scale=0.4]{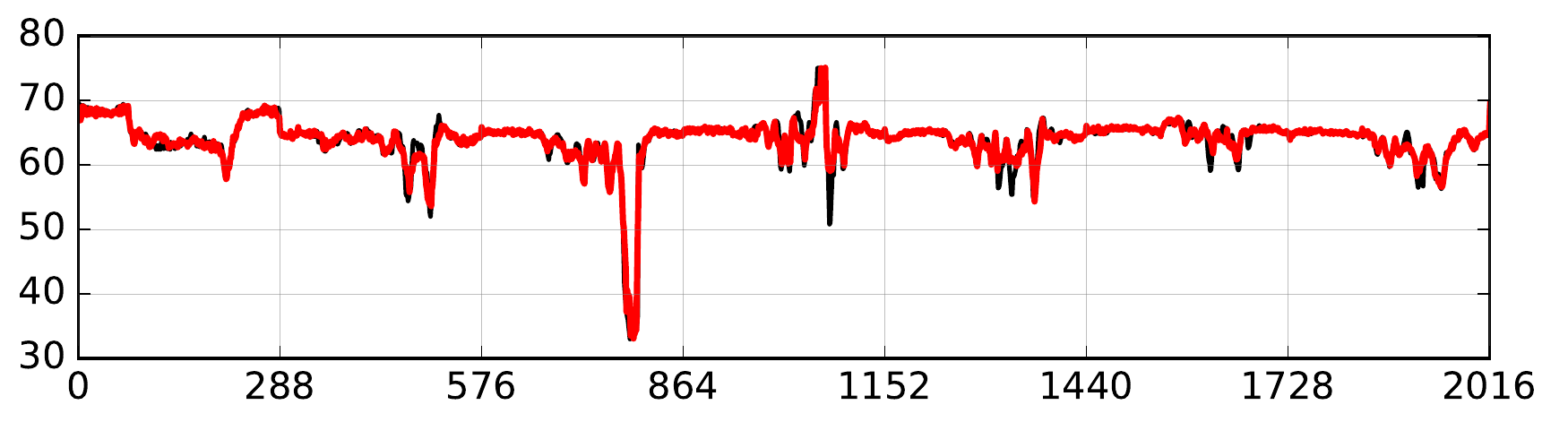}
}
\caption{Estimated time series (red curve) vs. ground truth (black curve) of the LSTC-Tubal model for PeMS-4W data at the case of 70\% RM scenario. Only the data in the first week are selected as examples. The first 6 time series correspond to the the first 6 rows in the data matrix.}
\label{time_series_curves}
\end{figure*}

\subsubsection{Evaluation on London-1M and Guangzhou-2M Data}

Table~\ref{table1_add} shows the imputation performance of LSTC-Tubal and its competing models on London-1M and Guangzhou-2M data sets. Of these results, LSTC-Tubal performs better than HaLRTC and LSTC-DCT. It implies that the LSTC framework with unitary transform is superior to both tensor nuclear norm minimization in HaLRTC and LSTC with discrete cosine transform. On London-1M data, LSTC-Tubal achieves competitive results as the best one reported by LRTC-TNN. On Guangzhou-2M data, LSTC-Tubal is approaching the best performance achieved by LRTC-TNN. In the NM scenario, the results of LSTC-Tubal are also comparable to the baseline models. Notably, quadratic variation smoothing can help improve the imputation performance of LSTC-Tubal.

\begin{table}[!ht]
\caption{Performance comparison (in MAPE/RMSE) for RM and NM data imputation tasks on London 1-month movement speed data (London-1M) and Guangzhou 2-month traffic speed data (Guangzhou-2M).}
\label{table1_add}
\centering
\footnotesize
\begin{tabular}{l|rrrr|rrrr}
\toprule
& \multicolumn{4}{c|}{London-1M} & \multicolumn{4}{c}{Guangzhou-2M} \\
\cmidrule(lr){2-5}
\cmidrule(lr){6-9}
& 30\%, RM & 70\%, RM & 30\%, NM & 70\%, NM & 30\%, RM & 70\%, RM & 30\%, NM & 70\%, NM \\
\midrule
BPMF & 9.11/2.20 & \textbf{9.40}/\textbf{2.27} & \textbf{9.39}/2.28 & \textbf{9.98}/\textbf{2.42} & 9.57/4.07 & 10.55/4.44 & 10.50/4.39 & 11.26/4.76 \\
BGCP & 9.19/2.24 & 9.44/2.30 & 9.46/2.31 & 9.99/2.44 & 8.33/3.59 & 8.54/3.69 & 10.47/4.37 & 10.74/4.66 \\
BATF & 9.17/2.23 & 9.45/2.30 & 9.44/2.31 & 9.99/2.44 & 8.29/3.58 & 8.53/3.69 & 10.39/4.34 & 10.75/4.55 \\
HaLRTC & 8.95/2.13 & 9.76/2.34 & 9.62/2.29 & 11.08/2.67 & 8.50/3.48 & 10.44/4.18 & 10.82/4.35 & 12.60/4.94 \\
LRTC-TNN & \textbf{8.87}/\textbf{2.12} & 9.63/2.32 & 9.46/\textbf{2.26} & 10.28/2.49 & \textbf{7.02}/\textbf{3.00} & \textbf{8.42}/\textbf{3.60} & \textbf{9.65}/\textbf{4.09} & \textbf{10.15}/\textbf{4.30} \\
LSTC-DCT & 9.41/2.27 & 10.15/2.45 & 10.50/2.52 & 11.74/2.85 & 7.73/3.25 & 9.50/3.93 & 11.64/4.72 & 12.64/5.13 \\
LSTC-Tubal ($\lambda=0$) & 9.14/2.20 & 9.73/2.34 & 9.87/2.37 & 11.00/2.66 & 7.50/3.18 & 8.84/3.71 & 10.78/4.42 & 11.19/4.56 \\
LSTC-Tubal ($\lambda\neq 0$) & 9.14/2.20 & 9.74/2.34 & 9.87/2.37 & 11.00/2.66 & {7.33}/{3.11} & {8.60}/{3.62} & 10.72/4.39 & 11.20/4.56 \\
\bottomrule
\end{tabular}
\end{table}

In particular, we discuss the imputation performance of the selected models, including BPMF, BGCP, HaLRTC, LRTC-TNN, and LSTC-Tubal, by analyzing the deviations between the ground truth data and imputed data. Figure~\ref{residual_distribution} shows the probability distribution of residuals $y_{i}-\hat{y}_{i}$ (i.e., ground truth minus imputed data) on the London-1M and Guangzhou-2M data sets. In Figure~\ref{residual_distribution}(a), these models have similar residual distribution and this indication is consistent with the imputation results as shown in Table~\ref{table1_add}. In Figure~\ref{residual_distribution}(b), the highest distribution peak around 0 is achieved by LSTC-Tubal, and it implies that there are more imputed values have ``zero errors" (i.e., $y_{i}-\hat{y}_{i}\approx0$). In this case, LSTC-Tubal performs better than baseline models.


\begin{figure*}[!ht]
\centering
\subfigure[On 70\% RM London-1M data.]{
    \centering
    \includegraphics[scale=0.45]{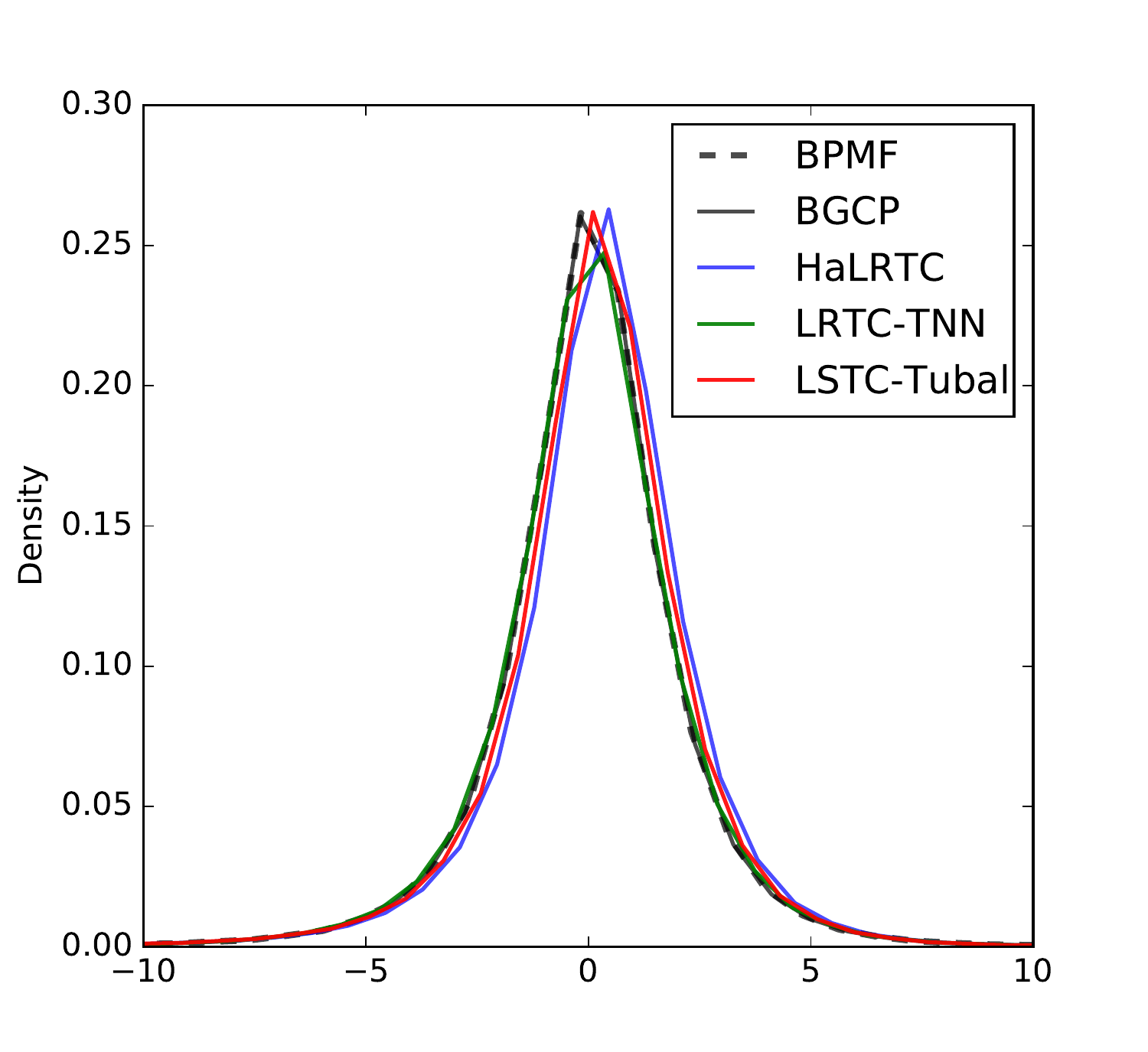}
}
\hspace{3em}
\subfigure[On 70\% RM Guangzhou-2M data.]{
    \centering
    \includegraphics[scale=0.45]{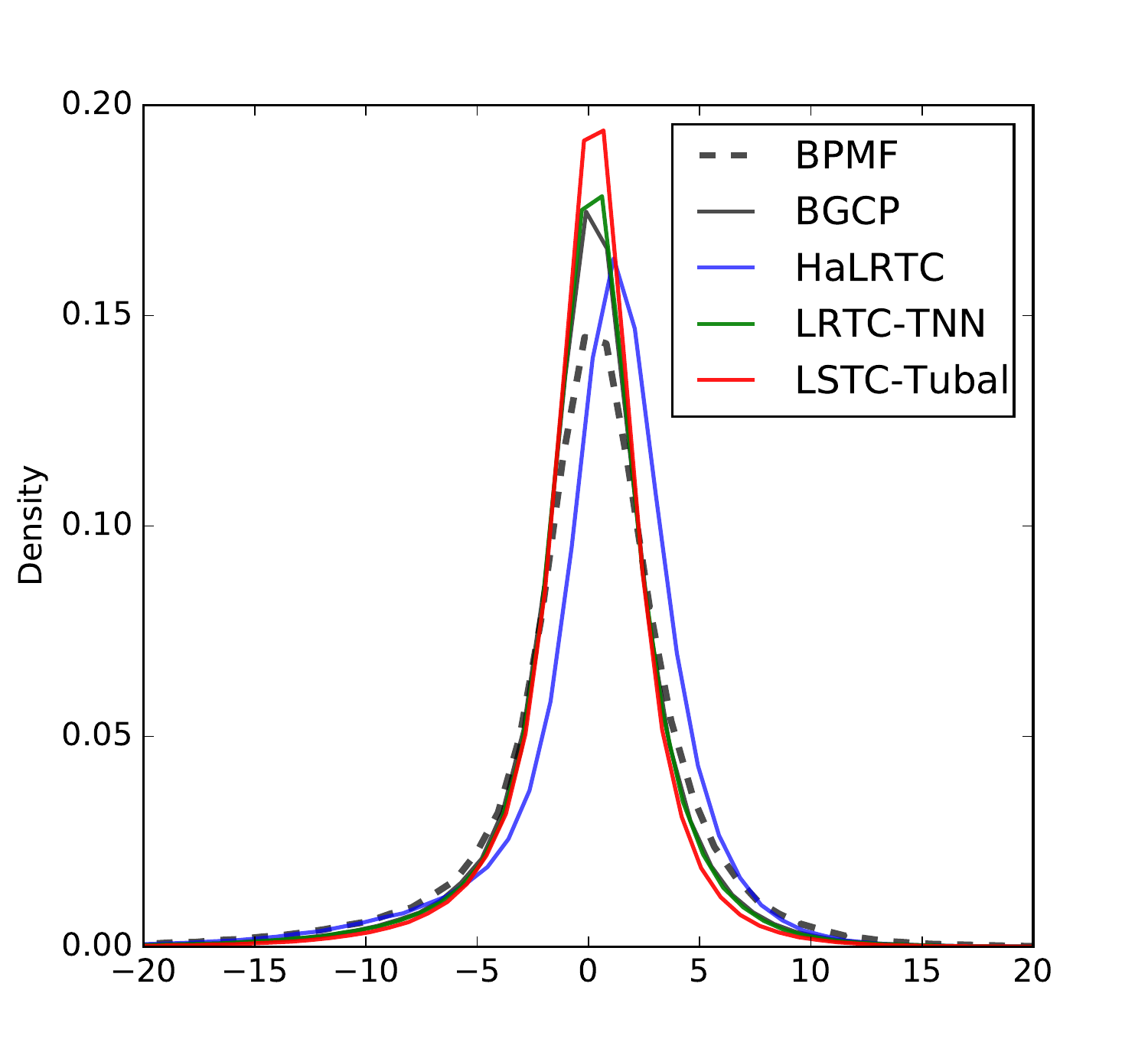}
}
\caption{Distribution of residuals between ground truth data and imputed data (i.e., $y_i-\hat{y}_{i}$) by using LSTC-Tubal and baseline models.}
\label{residual_distribution}
\end{figure*}

\subsubsection{Necessity of Imputation on Large Data}

As mentioned above, LSTC-Tubal model has significant superiority over other models for handling large-scale traffic data. Scalability is an important capability of models to handle data available over different scales.
Yet, one can wonder whether it is necessary to perform imputation on large-scale data. To answer this question, we conduct some experiments on PeMS-4W data. Recall that PeMS-4W data has 11160 road segments. To perform imputation on different scales of data, we use the multilevel $k$-way graph partitioning algorithm on the adjacency matrix of PeMS-4W data as in \cite{mallick2020graph} and set the number of partitions for PeMS-4W data as $\{2,4,8,16,32,64\}$. Table~\ref{subset_splitting} shows the imputation performance of BGCP and LSTC-Tubal with different missing scenarios and numbers of partitions. As can be seen:
\begin{itemize}
    \item BGCP (low rank: 10): In the RM scenario, MAPE/RMSE decreases with the increasing numbers of partitions, and BGCP achieves best imputation performance when the number of partitions is 64. However, in the 30\% NM scenario, the best imputation performance is achieved when the number of partitions is 8/16. And in the 70\% NM scenario, the best performance is achieved when the number of partitions is 1/2.
    \item LSTC-Tubal: In the RM scenario, LSTC-Tubal performs best when there is no partitioning. However, in the NM scenario, MAPE/RMSE decreases with the increasing number of partitions.
\end{itemize}

\begin{figure*}[!ht]
\centering
\subfigure[On 70\% RM PeMS-4W data.]{
    \centering
    \includegraphics[width = \textwidth]{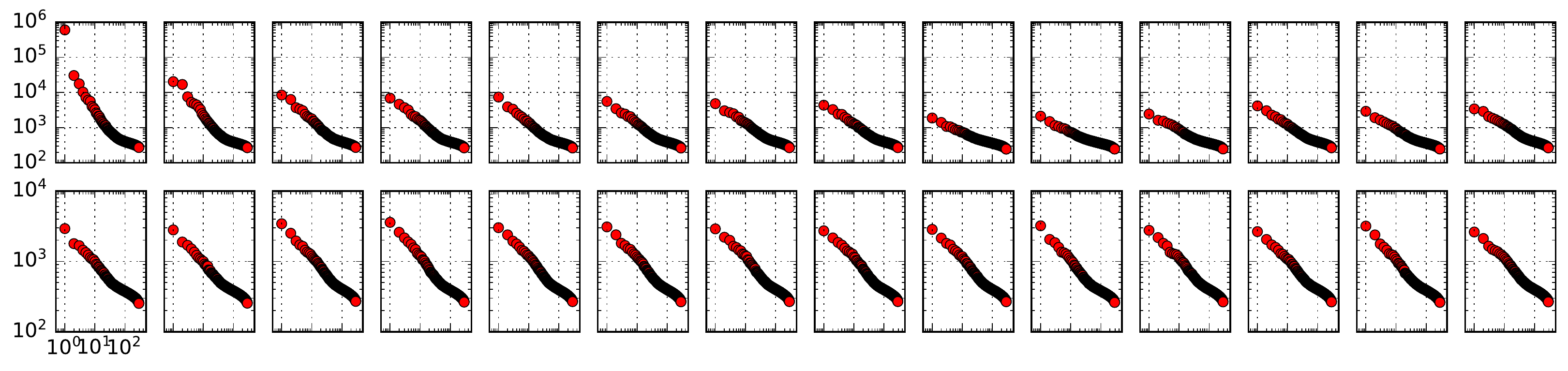}
}
\subfigure[On 70\% NM PeMS-4W data.]{
    \centering
    \includegraphics[width = \textwidth]{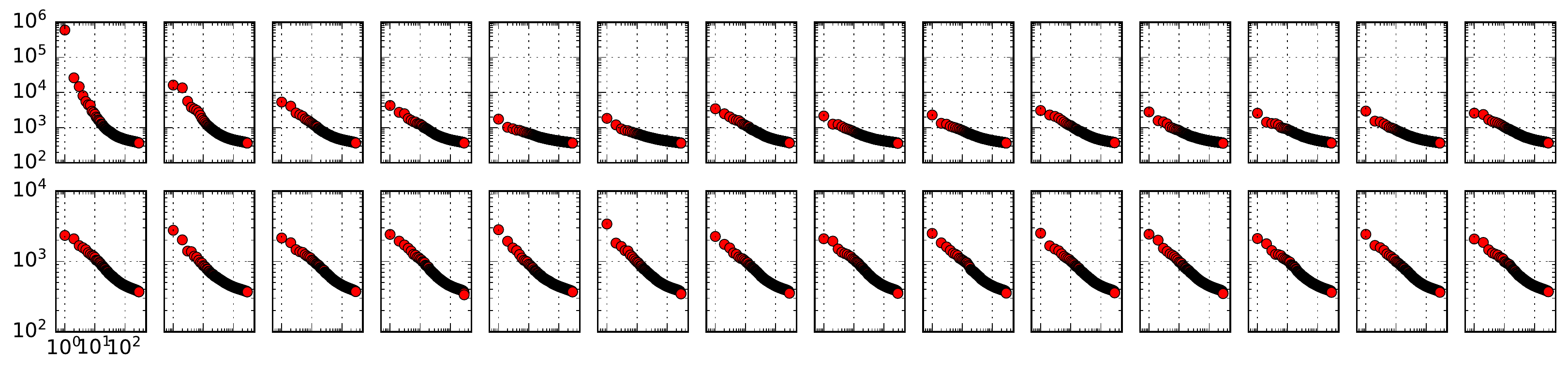}
}
\caption{Singular value scatters showing each $11160$-by-$288$ unitary transformed matrix (28 matrices in total) without graph partitioning. Note that singular values are sorted in decreasing orders ($\sigma_1\geq\sigma_2\geq\cdots$).}
\label{singular_values_case1}
\end{figure*}



Of these results, BGCP and LSTC-Tubal have different imputation performance. BGCP---a fully Bayesian tensor decomposition model with a fixed low rank---can preserve a more informative factorization structure for the RM scenario on small-scale data than on large-scale data due to the possibly strong similarity within each partition. However, in the NM scenario, learning a well-behaved BGCP model requires more informative data and therefore, BGCP performs better when there are fewer partitions.

Figure~\ref{singular_values_case1} shows that LSTC-Tubal model reports greater singular values for RM data than NM data. Compare to the NM scenario, there are more relatively large singular values obtained from LSTC-Tubal in the RM scenario. Table~\ref{subset_splitting} also demonstrates that LSTC-Tubal can learn informative singular values from large-scale data than from a series of subsets in the RM scenario. In the NM scenario, it seems that LSTC-Tubal can learn a series of effective low-rank structures from subsets than from large-scale data. In light of the above, there are two important factors that motivate the necessity of working on large-scale data: the pattern of missing data and the modeling mechanism. Therefore, modeling on large-scale data in this work is of great significance.

\begin{table}[!ht]
\caption{Imputation performance (in MAPE/RMSE) of BGCP and LSTC-Tubal for RM and NM data imputation tasks on PeMS-4W data with different numbers of partitions.}
\label{subset_splitting}
\centering
\footnotesize
\begin{tabular}{l|l|ccccccc}
\toprule
& & \multicolumn{7}{c}{Number of partitions} \\
\cmidrule(lr){3-9}
& & 1 & 2 & 4 & 8 & 16 & 32 & 64 \\
\midrule
\multirow{4}{*}{BGCP} & 30\%, RM & 5.03/4.31 & 4.94/4.26 & 4.93/4.23 & 4.82/4.16 & 4.75/4.11 & 4.59/3.98 & \textbf{4.42}/\textbf{3.84} \\
& 70\%, RM & 5.02/4.31 & 4.95/4.26 & 4.92/4.24 & 4.82/4.17 & 4.75/4.10 & 4.62/4.00 & \textbf{4.45}/\textbf{3.87} \\
& 30\%, NM & 5.38/4.59 & 5.35/4.56 & 5.26/4.50 & 5.20/\textbf{4.47} & \textbf{5.14}/4.50 & 5.22/5.73 & 5.86/24.56 \\
& 70\%, NM & \textbf{6.10}/\textbf{5.50} & 6.15/\textbf{5.50} & 6.22/5.97 & 6.96/10.25 & 8.61/30.80 & 9.29/38.39 & 13.60/72.51 \\
\midrule
\multirow{4}{*}{LSTC-Tubal} & 30\%, RM & \textbf{1.72}/\textbf{1.61} & 1.73/\textbf{1.61} & 1.75/1.62 & 1.82/1.65 & 1.91/1.72 & 2.00/1.79 & 2.10/1.89 \\
& 70\%, RM & \textbf{2.47}/\textbf{2.27} & 2.56/2.32 & 2.66/2.37 & 2.68/2.37 & 2.79/2.45 & 2.90/2.54 & 3.04/2.65 \\
& 30\%, NM & 5.59/4.52 & 5.52/4.47 & 5.39/4.38 & 5.27/4.29 & 5.15/4.20 & 4.94/4.05 & \textbf{4.69}/\textbf{3.89} \\
& 70\%, NM & 6.60/5.07 & 6.57/5.05 & 6.51/5.02 & 6.43/4.98 & 6.38/4.96 & 6.25/4.88 & \textbf{6.14}/\textbf{4.83} \\
\bottomrule
\end{tabular}
\end{table}

\section{Conclusion and Future Directions}\label{conclusion}

In this work, we develop a Low-Tubal-Rank Smoothing Tensor Completion (LSTC-Tubal) model for large-scale spatiotemporal imputation with superior efficiency. In this model, we use tensor nuclear norm minimization scheme to model the inherent low-rank property of traffic data (third-order tensors of sensor $\times$ time of day $\times$ day). This third-order tensor structure allows us to efficiently implement the nuclear norm minimization on the unitary transformed matrix for each day, while the day-to-day correlations can be preserved by the unitary transform matrix. Our experiments show that the proposed LSTC-Tubal model is very computationally efficient while still maintaining the imputation accuracy close to the state-of-the-art models. Experiments also show that the data-driven unitary transform matrix can harness effective traffic patterns.


This paper also provides valuable insight into spatiotemporal data modeling, and there are also some future directions to advance this research:
\begin{itemize}
    \item As shown in Lemma~\ref{tensor_svt}, the tensor singular value thresholding scheme is achieved by the singular value thresholding of matrix nuclear norm minimization. In the future, it would be feasible to adopt the truncated nuclear norm minimization for this scheme and improve the model performance.
    \item In this work, we only use the quadratic variation to smooth time series. It is possible to replace the quadratic variation by time series models for capturing temporal dynamics. In this way, the built model can be applied to not only missing traffic data imputation, but also traffic forecasting in the presence of missing values.
\end{itemize}

\section*{Acknowledgement}

This research is supported by the Natural Sciences and Engineering Research Council (NSERC) of Canada, the Fonds de recherche du Quebec – Nature et technologies (FRQNT), and the Canada Foundation for Innovation (CFI). X. Chen would  like  to  thank  the  Institute  for  Data  Valorisation (IVADO) for providing the Ph.D.\ scholarship to support this study.

\appendix



\section{Computing sources}

For the experimental evaluation, we used a platform consisting of a 3.5 GHz Intel Core i5-6600K processor (4 cores in total) and 32 GB of RAM. It is also possible to run our code in a laptop with relatively smaller RAM. In the settings of software, we used Python 3.7 (with the Numpy and Pandas packages).

\bibliographystyle{elsarticle-harv}
\bibliography{tensor}

\end{document}